\documentclass{article}

\usepackage{times}
\usepackage{graphicx} %

\usepackage{amsmath}
\usepackage{mathrsfs}
\usepackage{amsfonts}
\usepackage{amsthm}
\usepackage{stmaryrd}
\usepackage{xcolor}
\usepackage{colortbl}
\usepackage{cases}
\usepackage{subcaption}
\usepackage{multirow}
\usepackage{chngcntr}
\usepackage{enumitem}
\usepackage{mathtools}
\usepackage{booktabs}

\usepackage{natbib}

\usepackage{algorithm}
\usepackage{algorithmic}
  
\newcommand\blfootnote[1]{%
  \begingroup
  \renewcommand\thefootnote{}\footnote{#1}%
  \addtocounter{footnote}{-1}%
  \endgroup
}

\usepackage{hyperref}

\usepackage[accepted]{icml2017}

\definecolor{emerald}{HTML}{2ecc71}
\definecolor{amethyst}{HTML}{9b59b6}
\definecolor{sunflower}{HTML}{f1c40f}
\definecolor{pumpkin}{HTML}{d35400}
\definecolor{concrete}{HTML}{95a5a6}
\definecolor{turquoise}{HTML}{1abc9c}
\definecolor{pomegranate}{HTML}{c0392b}
\definecolor{belizehole}{HTML}{2980b9}
\definecolor{wetasphalt}{HTML}{34495e}

\definecolor{lightsunflower}{HTML}{FFEC67}
\definecolor{cloud}{HTML}{ecf0f1}

\newcommand{\pnorm}[1]{\| #1 \|_p} 
\newcommand{\panorm}[1]{| #1 |^p} 

\newcommand{\cX}{\mathcal{X}}

\newcommand{\cF}{\mathcal{F}}
\newcommand{\cA}{\mathcal{A}}

\newcommand{\cZ}{\mathcal{Z}}

\newcommand{\cT}{\mathcal{T}}
\newcommand{\cG}{\mathcal{G}}

\newcommand{\cL}{\mathcal{L}}

\newcommand{\bR}{\mathbb{R}}

\newcommand{\bN}{\mathbb{N}}

\newcommand{\indic}[1]{\mathbb{I}\left [ #1 \right ]}

\DeclareMathOperator*{\expect}{{\huge \mathbb{E}}}
\newcommand{\expects}{\expect\nolimits}

\newcommand{\Var}{\mathbb{V}}

\newcommand{\infnorm}[1]{\left \| #1 \right \|_{\infty}}

\newtheorem{defn}{Definition}
\newtheorem{prop}{Proposition}
\newtheorem{lem}{Lemma}

\newtheorem{thm}{Theorem}

\newcommand{\cbar}{\, | \,}
\newcommand{\cdbar}{\, \| \,}

\DeclareMathOperator*{\argmax}{arg\,max}

\newcommand{\eqnref}[1]{(\ref{eqn:#1})}

\icmltitlerunning{A Distributional Perspective on Reinforcement Learning}

\def \Ddef {\overset{D}{:=}}
\def \Deq {\overset{D}{=}}

\def \cTpi {\cT^\pi}

\def \dip {\bar d_p}
\def \Vrange {B}
\def \cXk {\cX_k}

\def \Vmax {V_{\textsc{max}}}
\def \Vmin {V_{\textsc{min}}}

\def \bu {\mathbf{u}}

\begin{document} 

\twocolumn[
\icmltitle{A Distributional Perspective on Reinforcement Learning}

\icmlsetsymbol{equal}{*}

\begin{icmlauthorlist}
\icmlauthor{Marc G. Bellemare}{equal,to}
\icmlauthor{Will Dabney}{equal,to}
\icmlauthor{R\'emi Munos}{to}
\end{icmlauthorlist}

\icmlaffiliation{to}{DeepMind, London, UK}
\icmlcorrespondingauthor{Marc G. Bellemare}{bellemare@google.com}

\vskip 0.3in
] 
\printAffiliationsAndNotice{\icmlEqualContribution} %

\begin{abstract}
In this paper we argue for the fundamental importance of the \emph{value distribution}: the distribution of the random return received by a reinforcement learning agent. This is in contrast to the common approach to reinforcement learning which models the expectation of this return, or \emph{value}. Although there is an established body of literature studying the value distribution, thus far it has always been used for a specific purpose such as implementing risk-aware behaviour.
We begin with theoretical results in both the policy evaluation and control settings, exposing a significant distributional instability in the latter.
We then use the distributional perspective to design a new algorithm which applies Bellman's equation to the learning of approximate value distributions.
We evaluate our algorithm using the suite of games from the Arcade Learning Environment. We obtain both state-of-the-art results and anecdotal evidence demonstrating the importance of the value distribution in approximate reinforcement learning.
Finally, we combine theoretical and empirical evidence to highlight the ways in which the value distribution impacts learning in the approximate setting.
\end{abstract} 

\section{Introduction}
One of the major tenets of reinforcement learning states that, when not otherwise constrained in its behaviour, an agent should aim to maximize its expected utility $Q$, or \emph{value} \citep{sutton98reinforcement}. Bellman's equation succintly describes this value in terms of the expected reward and expected outcome of the random transition $(x, a) \hspace{-0.1em} \to \hspace{-0.1em}  (X', A')$:
\begin{equation*}
Q(x,a) = \expect R(x,a) + \gamma \expect Q(X', A') .
\end{equation*}
In this paper, we aim to go beyond the notion of value and argue in favour of a distributional perspective on reinforcement learning.
Specifically, the main object of our study is the random return $Z$ whose expectation is the value $Q$. This random return is also described by a recursive equation, but one of a distributional nature:
\begin{equation*}
Z(x,a) \Deq R(x,a) + \gamma Z(X', A').
\end{equation*}
The \emph{distributional Bellman equation} states that the distribution of $Z$ is characterized by the interaction of three random variables: the reward $R$, the next state-action $(X', A')$, and its random return $Z(X', A')$. By analogy with the well-known case, we call this quantity the \emph{value distribution}.

Although the distributional perspective is almost as old as Bellman's equation itself \citep{jaquette73markov,sobel82variance,white88mean}, in reinforcement learning it has thus far been subordinated to specific purposes: to model parametric uncertainty \citep{dearden98bayesian}, to design risk-sensitive algorithms \citep{morimura2010nonparametric,morimura10parametric}, or for theoretical analysis \citep{azar12sample,lattimore12pac}.
By contrast, we believe the value distribution has a central role to play in reinforcement learning.

\textbf{Contraction of the policy evaluation Bellman operator.} Basing ourselves on results by \citet{rosler92fixed} we show that, for a fixed policy, the Bellman operator over value distributions is a contraction in a maximal form of the Wasserstein (also called Kantorovich or Mallows) metric. Our particular choice of metric matters: the same operator is not a contraction in total variation, Kullback-Leibler divergence, or Kolmogorov distance.

\textbf{Instability in the control setting.} We will demonstrate an instability in the distributional version of Bellman's optimality equation, in contrast to the policy evaluation case. Specifically, although the optimality operator is a contraction in expected value (matching the usual optimality result), it is not a contraction in any metric over distributions.
These results provide evidence in favour of learning algorithms that model the effects of nonstationary policies.

\textbf{Better approximations.} From an algorithmic standpoint, there are many benefits to learning an approximate distribution rather than its approximate expectation.
The distributional Bellman operator preserves multimodality in value distributions, which we believe leads to more stable learning. Approximating the full distribution also mitigates the effects of learning from a nonstationary policy. As a whole, we argue that this approach makes approximate reinforcement learning significantly better behaved.%

We will illustrate the practical benefits of the distributional perspective in the context of the Arcade Learning Environment \citep{bellemare13arcade}. By modelling the value distribution within a DQN agent \citep{mnih15nature}, we obtain considerably increased performance across the gamut of benchmark Atari 2600 games, and in fact achieve state-of-the-art performance on a number of games. Our results echo those of \citet{veness15compress}, who obtained extremely fast learning by predicting Monte Carlo returns.

From a supervised learning perspective, learning the full value distribution might seem obvious: why restrict ourselves to the mean? The main distinction, of course, is that in our setting there are no given targets. Instead, we use Bellman's equation to make the learning process tractable; we must, as \citet{sutton98reinforcement} put it, ``learn a guess from a guess''. It is our belief that this guesswork ultimately carries more benefits than costs.

\section{Setting}

We consider an agent interacting with an environment in the standard fashion: at each step, the agent selects an action based on its current state, to which the environment responds with a reward and the next state. We model this interaction as a time-homogeneous Markov Decision Process $(\cX, \cA, R, P, \gamma)$. As usual, $\cX$ and $\cA$ are respectively the state and action spaces, $P$ is the transition kernel $P(\cdot \cbar x, a)$, $\gamma \in [0, 1]$ is the discount factor, and
$R$ is the reward function, which in this work we explicitly treat as a random variable. A stationary policy $\pi$ maps each state $x \in \cX$ to a probability distribution over the action space $\cA$.

\subsection{Bellman's Equations}

The \emph{return} $Z^\pi$ is the sum of discounted rewards along the agent's trajectory of interactions with the environment. The value function $Q^\pi$ of a policy $\pi$ describes the expected return from taking action $a \in \cA$ from state $x \in \cX$, then acting according to $\pi$:
\begin{align}
Q^\pi(x,a) &:= \expect Z^\pi(x,a) = \expect \left [ \sum_{t=0}^\infty \gamma^t R(x_t, a_t) \right ], \label{eqn:random_return} \\
x_t \sim P(\cdot &\cbar x_{t-1}, a_{t-1}), a_t \sim \pi(\cdot \cbar x_t), x_0 = x, a_0 = a . \nonumber
\end{align}
Fundamental to reinforcement learning is the use of Bellman's equation  \citep{bellman57dynamic} to describe the value function:
\begin{equation*}
Q^\pi(x,a) = \expect R(x,a) + \gamma \expect_{P, \pi} Q^\pi(x', a') .
\end{equation*}
In reinforcement learning we are typically interested in acting so as to maximize the return. The most common approach for doing so involves the optimality equation
\begin{equation*}
Q^*(x,a) = \expect R(x,a) + \gamma \expect\nolimits_P \max_{a' \in \cA} Q^*(x', a') .
\end{equation*}
This equation has a unique fixed point $Q^*$, the optimal value function, corresponding to the set of optimal policies $\Pi^*$ ($\pi^*$ is optimal if $\expect_{a \sim \pi^*} Q^*(x, a) = \max_a Q^*(x,a)$).

We view value functions as vectors in $\bR^{\cX \times \cA}$, and the expected reward function as one such vector. In this context, the \emph{Bellman operator} $\cT^\pi$ and \emph{optimality operator} $\cT$ are
\begin{align}
\cT^\pi Q(x,a) &:= \expect R(x,a) + \gamma \expect_{P, \pi} Q(x',a')\label{eqn:bellman_operator_pe} \\
\cT Q(x,a) &:= \expect R(x,a) + \gamma \expects_{P} \max_{a' \in \cA} Q(x', a') . \label{eqn:bellman_operator_optimality}
\end{align}
These operators are useful as they describe the expected behaviour of popular learning algorithms such as SARSA and Q-Learning. In particular they are both contraction mappings, and their repeated application to some initial $Q_0$ converges exponentially to $Q^\pi$ or $Q^*$, respectively \citep{bertsekas96neurodynamic}.

\section{The Distributional Bellman Operators}\label{sec:theory}
\begin{figure}
\begin{center}
\includegraphics[width=0.48\textwidth]{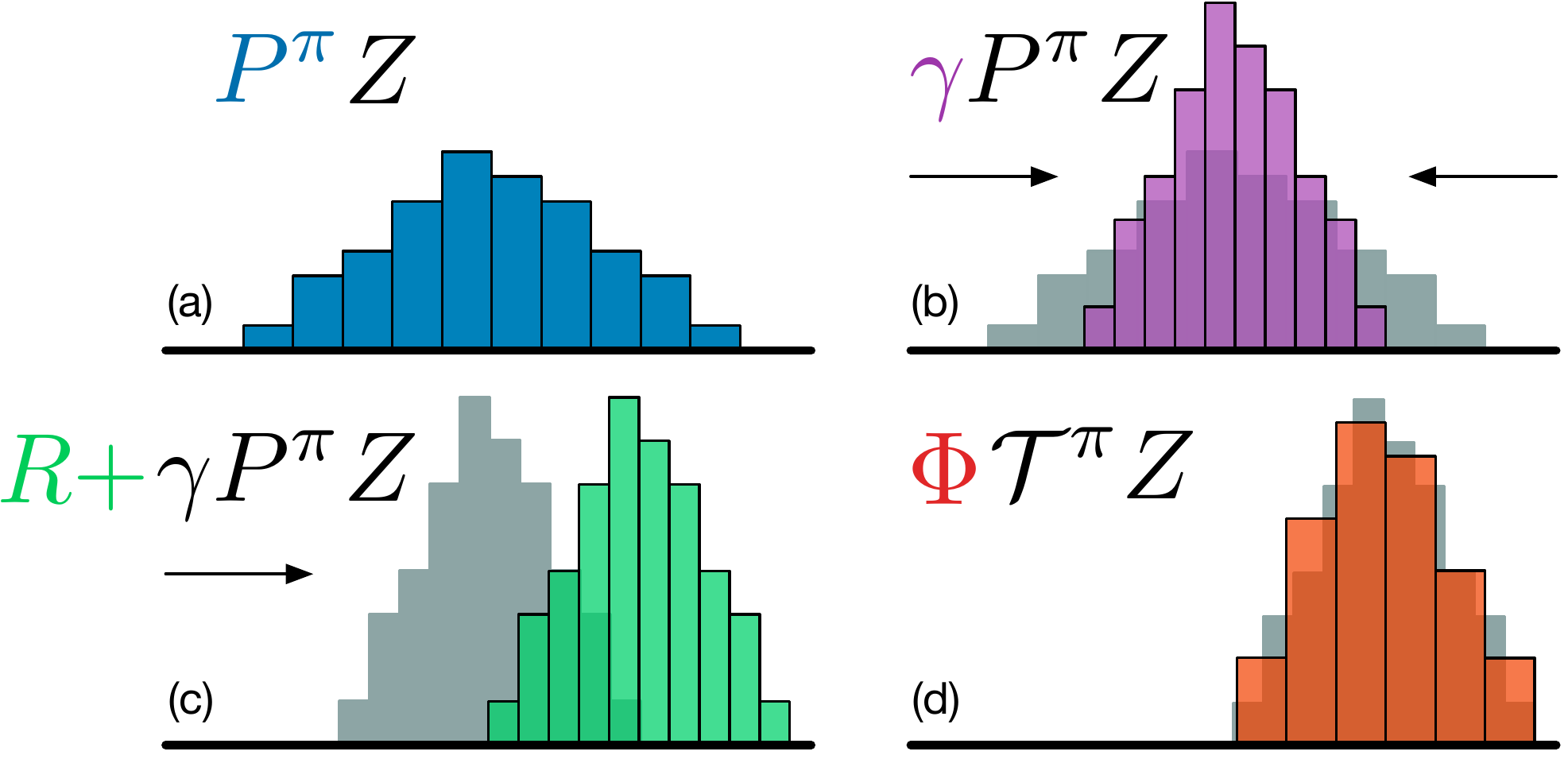}
\end{center}
\caption{A distributional Bellman operator with a deterministic reward function: (a) Next state distribution under policy $\pi$, (b) Discounting shrinks the distribution towards 0, (c) The reward shifts it, and (d) Projection step (Section \ref{sec:algorithm}). \label{fig:bellman_dist}}
\end{figure}

In this paper we take away the expectations inside Bellman's equations and consider instead the full distribution of the random variable $Z^\pi$. From here on, we will view $Z^\pi$ as a mapping from state-action pairs to distributions over returns, and call it the \emph{value distribution}. 

Our first aim is to gain an understanding of the theoretical behaviour of the distributional analogues of the Bellman operators, in particular in the less well-understood control setting. The reader strictly interested in the algorithmic contribution may choose to skip this section.

\subsection{Distributional Equations}

It will sometimes be convenient to make use of the probability space $(\Omega, \cF, \Pr)$. The reader unfamiliar with measure theory may think of $\Omega$ as the space of all possible outcomes of an experiment \citep{billingsley95probability}.
We will write $\pnorm{\bu}$ to denote the $L_p$ norm of a vector $\bu \in \bR^\cX$ for $1 \le p \le \infty$; the same applies to vectors in $\bR^{\cX \times \cA}$. The $L_p$ norm of a random vector $U : \Omega \to \bR^\cX$ (or $\bR^{\cX \times \cA}$) is then $\pnorm{U} := \left [ \expect \left [ \pnorm{U(\omega)}^p \right ] \right ]^{1/p}$, and for $p = \infty$ we have $\| U \|_\infty = \textrm{ess}\sup{\|U(\omega)\|_\infty}$ (we will omit the dependency on $\omega \in \Omega$ whenever unambiguous).
We will denote the c.d.f. of a random variable $U$ by $F_U(y) := \Pr\{ U \le y \}$, and its inverse c.d.f. by $F^{-1}_U(q) := \inf \{ y : F_U(y) \ge q \}$.

A distributional equation $U \Ddef V$ indicates that the random variable $U$ is distributed according to the same law as $V$. Without loss of generality, the reader can understand the two sides of a distributional equation as relating the distributions of two independent random variables.
Distributional equations have been used in reinforcement learning by \citet{engel05reinforcement,morimura10parametric} among others, and in operations research by \citet{white88mean}.

\subsection{The Wasserstein Metric}

The main tool for our analysis is the Wasserstein metric $d_p$ between cumulative distribution functions \citep[see e.g.][where it is called the Mallows metric]{bickel81asymptotic}. For $F$, $G$ two c.d.fs over the reals, it is defined as
\begin{equation*}
d_p(F, G) := \inf_{U, V} \pnorm{U - V},
\end{equation*}
where the infimum is taken over all pairs of random variables $(U, V)$ with respective cumulative distributions $F$ and $G$. The infimum is attained by the inverse c.d.f. transform of a random variable $\mathcal{U}$ uniformly distributed on $[0, 1]$:
\begin{equation*}
d_p(F, G) = \| F^{-1}(\mathcal{U}) - G^{-1}(\mathcal{U}) \|_p .
\end{equation*}
For $p < \infty$ this is more explicitly written as
\begin{equation*}
d_p(F, G) = \left ( \int_0^1 \big | F^{-1}(u) - G^{-1}(u) \big |^p du \right )^{1/p} .
\end{equation*}
Given two random variables $U$, $V$ with c.d.fs $F_U$, $F_V$, we will write $d_p(U, V) := d_p(F_U, F_V)$. We will find it convenient to conflate the random variables under consideration with their versions under the $\inf$, writing
\begin{equation*}
d_p(U, V) = \inf_{U,V} \pnorm{U - V} .
\end{equation*}
whenever unambiguous; we believe the greater legibility justifies the technical inaccuracy.
Finally, we extend this metric to vectors of random variables, such as value distributions, using the corresponding $L_p$ norm.

Consider a scalar $a$ and a random variable $A$ independent of $U, V$. The metric $d_p$ has the following properties:
\begin{align*}
d_p(aU, aV) &\le |a| d_p(U, V) \tag{P1} \\
d_p(A + U, A + V) &\le d_p(U, V) \tag{P2} \\
d_p(AU, AV) & \le \pnorm{A} d_p(U, V). \tag{P3}
\end{align*}
We will need the following additional property, which makes no independence assumptions on its variables.
Its proof, and that of later results, is given in the appendix.
\begin{lem}[Partition lemma]\label{lem:partition_lemma}
Let $A_1, A_2, \dots$ be a set of random variables describing a partition of $\Omega$, i.e. $A_i(\omega) \in \{0, 1\}$ and for any $\omega$ there is exactly one $A_i$ with $A_i(\omega) = 1$. Let $U, V$ be two random variables. Then
\begin{equation*}
d_p\big (U, V \big) \le \sum\nolimits_i d_p (A_i U, A_i V) .
\end{equation*}
\end{lem}
Let $\cZ$ denote the space of value distributions with bounded moments.
For two value distributions $Z_1, Z_2 \in \cZ$ we will make use of a maximal form of the Wasserstein metric:
\begin{equation*}
\dip(Z_1, Z_2) := \sup_{x, a} d_p(Z_1(x,a), Z_2(x,a)) .
\end{equation*}
We will use $\dip$ to establish the convergence of the distributional Bellman operators.
\begin{lem}
$\dip$ is a metric over value distributions.
\end{lem}

\subsection{Policy Evaluation}\label{sec:policy_evaluation}

In the \emph{policy evaluation} setting \citep{sutton98reinforcement} we are interested in the value function $V^\pi$ associated with a given policy $\pi$. The analogue here is the value distribution $Z^\pi$. In this section we characterize $Z^\pi$ and study the behaviour of the policy evaluation operator $\cTpi$. We emphasize that $Z^\pi$ describes the intrinsic randomness of the agent's interactions with its environment, rather than some measure of uncertainty about the environment itself.

We view the reward function as a random vector $R \in \cZ$, and define the transition operator $P^\pi : \cZ \to \cZ$
\begin{align}
P^\pi Z(x, a) &\Ddef Z(X', A') \label{eqn:policy_operator} \\
X' &\sim P(\cdot \cbar x, a), \, A' \sim \pi(\cdot \cbar X'), \nonumber
\end{align}
where we use capital letters to emphasize the random nature of the next state-action pair $(X', A')$.
We define the distributional Bellman operator $\cTpi : \cZ \to \cZ$ as
\begin{equation}
\cTpi Z(x,a) \Ddef R(x,a) + \gamma P^\pi Z(x,a). \label{eqn:distributional_bellman_operator_pe}
\end{equation}
While $\cTpi$ bears a surface resemblance to the usual Bellman operator \eqnref{bellman_operator_pe}, it  is fundamentally different. In particular, three sources of randomness define the compound distribution $\cTpi Z$:
\begin{enumerate}[noitemsep,label=\alph*)]
	\item The randomness in the reward $R$,
	\item The randomness in the transition $P^\pi$, and
	\item The next-state value distribution $Z(X', A')$.
\end{enumerate}
In particular, we make the usual assumption that these three quantities are independent.
In this section we will show that \eqnref{distributional_bellman_operator_pe} is a contraction mapping whose  unique fixed point is the random return $Z^\pi$.

\subsubsection{Contraction in $\dip$}
Consider the process $Z_{k+1} := \cTpi Z_k$, starting with some $Z_0 \in \cZ$. We may expect the limiting expectation of $\{ Z_k \}$ to converge exponentially quickly, as usual, to $Q^\pi$. As we now show, the process converges in a stronger sense: $\cTpi$ is a contraction in $\dip$, which implies that all moments also converge exponentially quickly.

\begin{lem}\label{lem:contraction_pe}
$\cTpi : \cZ \to \cZ$ is a $\gamma$-contraction in $\dip$.
\end{lem}
Using Lemma \ref{lem:contraction_pe}, we conclude using Banach's fixed point theorem that $\cTpi$ has a unique fixed point. By inspection, this fixed point must be $Z^\pi$ as defined in \eqnref{random_return}. As we assume all moments are bounded, this is sufficient to conclude that the sequence $\{ Z_k \}$ converges to $Z^\pi$ in $\dip$ for $1 \le p \le \infty$.

To conclude, we remark that not all
distributional metrics are equal; for example, \citet{chung87discounted} have shown that $\cTpi$ is not a contraction in total variation distance. Similar results can be derived for the Kullback-Leibler divergence and the Kolmogorov distance.

\subsubsection{Contraction in Centered Moments}
Observe that $d_2(U, V)$ (and more generally, $d_p$) relates to a coupling $C(\omega) := U(\omega) - V(\omega)$, in the sense that
\begin{align*}
d_2^2(U, V) &\le \expect [ (U - V)^2 ] = \Var \big (C \big ) + \big ( \expect C \big )^2 .
\end{align*}
As a result, we cannot directly use $d_2$ to bound the variance difference $| \Var(\cTpi Z(x,a)) - \Var(Z^\pi(x,a)) |$. However, $\cTpi$ is in fact a contraction in variance \citep[][see also appendix]{sobel82variance}.
In general, $\cTpi$ is not a contraction in the $p^{th}$ centered moment, $p > 2$, but the centered moments of the iterates $\{ Z_k \}$ still converge exponentially quickly to those of $Z^\pi$; the proof extends the result of \citet{rosler92fixed}.

\subsection{Control}\label{sec:control}

Thus far we have considered a fixed policy $\pi$, and studied the behaviour of its associated operator $\cTpi$. We now set out to understand the distributional operators of the \emph{control} setting -- where we seek a policy $\pi$ that maximizes value -- and the corresponding notion of an optimal value distribution. As with the optimal value function, this notion is intimately tied to that of an optimal policy. However, while all optimal policies attain the same value $Q^*$, in our case a difficulty arises: in general there are many optimal value distributions.

In this section we show that the distributional analogue of the Bellman optimality operator converges, in a weak sense, to the set of optimal value distributions. However, this operator is not a contraction in any metric between distributions, and is in general much more temperamental than the policy evaluation operators. We believe the convergence issues we outline here are a symptom of the inherent instability of greedy updates, as highlighted by e.g.~\citet{tsitsiklis02convergence} and most recently \citet{harutyunyan16qlambda}.

Let $\Pi^*$ be the set of optimal policies. We begin by characterizing what we mean by an \emph{optimal value distribution}.
\begin{defn}[Optimal value distribution]
An optimal value distribution is the v.d. of an optimal policy. The set of optimal value distributions is $\cZ^* := \{ Z^{\pi^*} : \pi^* \in \Pi^* \}$.
\end{defn}
We emphasize that not all value distributions with expectation $Q^*$ are optimal: they must match the full distribution of the return under some optimal policy.
\begin{defn}
A greedy policy $\pi$ for $Z \in \cZ$ maximizes the expectation of $Z$. The set of greedy policies for $Z$ is
\begin{equation*}
\cG_Z := \{ \pi : \sum\nolimits_a \pi(a \cbar x) \expect Z(x, a) = \max_{a' \in \cA} \expect Z(x, a') \}.
\end{equation*}
\end{defn}
Recall that the expected Bellman optimality operator $\cT$ is
\begin{equation}
\cT Q(x,a) = \expect R (x,a) + \gamma \expects_P \max_{a' \in \cA} Q(x', a') . \label{eqn:bellman_optimality_operator}
\end{equation}
The maximization at $x'$ corresponds to some greedy policy. Although this policy is implicit in \eqnref{bellman_optimality_operator}, we cannot ignore it in the distributional setting.
We will call a \emph{distributional Bellman optimality operator} any operator $\cT$ which implements a greedy selection rule, i.e.
\begin{equation*}
\cT Z = \cTpi Z\text{ for some } \pi \in \cG_Z .
\end{equation*}
As in the policy evaluation setting, we are interested in the behaviour of the iterates $Z_{k+1} := \cT Z_k$, $Z_0 \in \cZ$. Our first result is to assert that $\expect Z_k$ behaves as expected.
\begin{lem}\label{lem:exponential_convergence_of_mean}
Let $Z_1, Z_2 \in \cZ$. Then
\begin{equation*}
\infnorm{\expect \cT Z_1 - \expect \cT Z_2} \le \gamma \infnorm{\expect Z_1 - \expect Z_2},
\end{equation*}
and in particular $\expect Z_k \to Q^*$ exponentially quickly.
\end{lem}
By inspecting Lemma \ref{lem:exponential_convergence_of_mean}, we might expect that $Z_k$ converges quickly in $\dip$ to some fixed point in $\cZ^*$. Unfortunately, convergence is neither quick nor assured to reach a fixed point. In fact, the best we can hope for is pointwise convergence, not even to the set $\cZ^*$ but to the larger set of \emph{nonstationary optimal value distributions}.
\begin{defn}
A nonstationary optimal value distribution $Z^{**}$ is the value distribution corresponding to a sequence of optimal policies. The set of n.o.v.d. is $\cZ^{**}$.
\end{defn}
\begin{thm}[Convergence in the control setting]\label{thm:control_convergence}
Let $\cX$ be measurable and suppose that $\cA$ is finite. Then
\begin{equation*}
\lim_{k \to \infty} \inf_{Z^{**} \in \cZ^{**}} d_p (Z_k(x,a), Z^{**}(x,a)) = 0 \quad \forall x, a .
\end{equation*}
If $\cX$ is finite, then $Z_k$ converges to $\cZ^{**}$ uniformly. Furthermore, if there is a total ordering $\prec$ on $\Pi^*$, such that for any $Z^* \in \cZ^*$,
\begin{equation*}
\cT Z^* = \cT^\pi Z^* \; \text{with} \; \pi \in \cG_{Z^*}, \; \pi \prec \pi' \; \; \forall \pi' \in \cG_{Z^*} \setminus \{\pi \} .
\end{equation*}
Then $\cT$ has a unique fixed point $Z^* \in \cZ^*$.
\end{thm}
Comparing Theorem \ref{thm:control_convergence} to Lemma \ref{lem:exponential_convergence_of_mean} reveals a significant difference between the distributional framework and the usual setting of expected return. While the mean of $Z_k$ converges exponentially quickly to $Q^*$, its distribution need not be as well-behaved! To emphasize this difference, we now provide a number of negative results concerning $\cT$.
\begin{prop}\label{prop:control_noncontraction}
The operator $\cT$ is not a contraction.
\end{prop}
Consider the following example (Figure \ref{fig:control_noncontraction}, left). There are two states, $x_1$ and $x_2$; a unique transition from $x_1$ to $x_2$; from $x_2$, action $a_1$ yields no reward, while the optimal action $a_2$ yields $1 + \epsilon$ or $-1 + \epsilon$ with equal probability. Both actions are terminal. There is a unique optimal policy and therefore a unique fixed point $Z^*$. Now consider $Z$ as given in Figure \ref{fig:control_noncontraction} (right), and its distance to $Z^*$:
\begin{equation*}
\bar d_1(Z, Z^*) = d_1(Z(x_2, a_2), Z^*(x_2, a_2)) = 2 \epsilon , \\
\end{equation*}
where we made use of the fact that $Z = Z^*$ everywhere except at $(x_2, a_2)$. When we apply $\cT$ to $Z$, however, the greedy action $a_1$ is selected and $\cT Z(x_1) = Z(x_2, a_1)$. But
\begin{align*}
d_1(\cT Z, \cT Z^*) &= d_1(\cT Z(x_1), Z^*(x_1)) \\
&= \tfrac{1}{2} | 1 - \epsilon | + \tfrac{1}{2}  | 1 + \epsilon | > 2 \epsilon
\end{align*}
for a sufficiently small $\epsilon$. This shows that the undiscounted update is not a nonexpansion: $\bar d_1(\cT Z, \cT Z^*) > \bar d_1(Z, Z^*)$. With $\gamma < 1$, the same proof shows it is not a contraction. Using a more technically involved argument, we can extend this result to any metric which separates $Z$ and $\cT Z$.

\begin{figure}
\begin{center}
\begin{minipage}{0.1\textwidth}
\hspace{-3em} \includegraphics[width=0.85in]{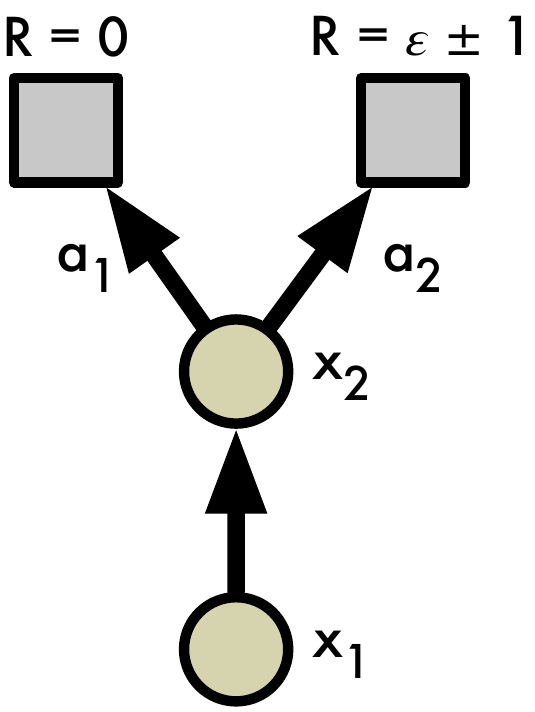}
\end{minipage}
\begin{minipage}{0.2\textwidth}
$\begin{array}[c]{lccc}
\multicolumn{1}{c}{} & \multicolumn{1}{c}{x_1} & \multicolumn{1}{c}{x_2, a_1} & \multicolumn{1}{c}{x_2, a_2} \\
\midrule
Z^* & \cellcolor{lightsunflower}\epsilon \pm 1 & 0 & \cellcolor{lightsunflower}\epsilon \pm 1 \\
\midrule
\midrule
Z & \epsilon \pm 1 & 0 & \cellcolor{lightsunflower} -\epsilon \pm 1 \\
\midrule
\cT Z & \cellcolor{lightsunflower}{0} & 0 & \epsilon \pm 1 \\
\bottomrule
\end{array}$
\end{minipage}
\end{center}
\caption{Undiscounted two-state MDP for which the optimality operator $\cT$ is not a contraction, with example. The entries that contribute to $\bar d_1(Z, Z^*)$ and $\bar d_1(\cT Z, Z^*)$ are highlighted.\label{fig:control_noncontraction}}
\end{figure}
\begin{prop}
Not all optimality operators have a fixed point $Z^* = \cT Z^*$.
\end{prop}
To see this, consider the same example, now with $\epsilon = 0$, and a greedy operator $\cT$ which breaks ties by picking $a_2$ if $Z(x_1) = 0$, and $a_1$ otherwise. Then the sequence $\cT Z^*(x_1), (\cT)^2 Z^*(x_1),\, \dots$ alternates between $Z^*(x_2, a_1)$ and $Z^*(x_2, a_2)$.
\begin{prop}
That $\cT$ has a fixed point $Z^* = \cT Z^*$ is insufficient to guarantee the convergence of $\{ Z_k \}$ to $\cZ^*$.
\end{prop}
Theorem \ref{thm:control_convergence} paints a rather bleak picture of the control setting. It remains to be seen whether the dynamical eccentricies highlighted here actually arise in practice. One open question is whether theoretically more stable behaviour can be derived using stochastic policies, for example from conservative policy iteration \citep{kakade02approximately}.

\section{Approximate Distributional Learning}\label{sec:algorithm}

In this section we propose an algorithm based on the distributional Bellman optimality operator. In particular, this will require choosing an approximating distribution. Although the Gaussian case has previously been considered \citep{morimura10parametric,tamar2016learning}, to the best of our knowledge we are the first to use a rich class of parametric distributions.

\subsection{Parametric Distribution}\label{sec:parametric_distribution}
We will model the value distribution using a discrete distribution parametrized by $N \in \bN$ and $\Vmin,\Vmax \in \bR$, and whose support is the set of atoms $\{z_i = \Vmin + i \triangle z : 0 \le i < N \}$, $\triangle z := \tfrac{\Vmax - \Vmin}{N-1}$. In a sense, these atoms are the ``canonical returns'' of our distribution. The atom probabilities are given by a parametric model $\theta: \mathcal{X} \times \mathcal{A} \to \mathbb{R}^N$
\begin{equation*}
	Z_\theta(x, a) = z_i \quad \text{w.p. } \; p_i(x, a) := \frac{e^{\theta_i(x, a)}}{\sum\nolimits_j e^{\theta_j(x, a)}}.
\end{equation*}
The discrete distribution has the advantages of being highly expressive and computationally friendly \citep[see e.g.][]{vandenoord16pixel}.

\subsection{Projected Bellman Update}
Using a discrete distribution poses a problem: the Bellman update $\cT Z_\theta$ and our parametrization $Z_\theta$ almost always have disjoint supports. From the analysis of Section \ref{sec:theory} it would seem natural to minimize the Wasserstein metric (viewed as a loss) between $\cT Z_\theta$ and $Z_\theta$, which is also conveniently robust to discrepancies in support.
However, a second issue prevents this: in practice we are typically restricted to learning from sample transitions, which is not possible under the Wasserstein loss (see Prop. \ref{prop:wasserstein_sgd_rl} and toy results in the appendix).

Instead, we project the sample Bellman update $\hat \cT Z_\theta$ onto the support of $Z_\theta$ (Figure \ref{fig:bellman_dist}, Algorithm \ref{alg:cql}), effectively reducing the Bellman update to multiclass classification. Let $\pi$ be the greedy policy w.r.t. $\expect Z_\theta$.
Given a sample transition $(x, a, r, x')$, we compute the Bellman update $\hat \cT z_j := r + \gamma z_j$ for each atom $z_j$, then distribute its probability $p_j(x', \pi(x'))$ to the immediate neighbours of $\hat \cT z_j$.
The $i^{th}$ component of the projected update $\Phi \hat \cT Z_\theta(x,a)$ is
\begin{small}
\begin{equation}
(\Phi \hat \cT Z_\theta(x,a))_i = \sum_{j=0}^{N-1} \left[ 1 - \frac{| [ \hat \cT z_j ]^{\Vmax}_{\Vmin} - z_i|}{\triangle z} \right]^1_0 p_j(x', \pi(x')),
\label{eqn:cat_proj}
\end{equation}
\end{small}
where $[ \cdot ]^b_a$ bounds its argument in the range $[a, b]$.\footnote{Algorithm \ref{alg:cql} computes this projection in time linear in $N$.} As is usual, we view the next-state distribution as parametrized by a fixed parameter $\tilde \theta$. The sample loss $\cL_{x,a}(\theta)$ is the cross-entropy term of the KL divergence
\begin{equation*}
D_{\textsc{kl}}(\Phi \hat \cT Z_{\tilde \theta}(x,a) \cdbar Z_\theta(x,a)),
\end{equation*}
which is readily minimized e.g. using gradient descent. We call this choice of distribution and loss the \emph{categorical algorithm}. When $N = 2$, a simple one-parameter alternative is
$\Phi \hat \cT Z_\theta(x,a) := [\expect [\hat \cT Z_\theta(x,a)] - \Vmin) / \triangle z]^1_0;$
we call this the \emph{Bernoulli algorithm}. 
We note that, while these algorithms appear unrelated to the Wasserstein metric, recent work \citep{bellemare17cramer} hints at a deeper connection.

\begin{algorithm}[ht]
\caption{Categorical Algorithm}\label{alg:cql}
\begin{algorithmic}
\INPUT A transition $x_t, a_t, r_t, x_{t+1}$, $\gamma_t \in [0, 1]$
\STATE $Q(x_{t+1}, a) := \sum\nolimits_i z_i p_i(x_{t+1}, a)$
\STATE $a^* \leftarrow \argmax_a Q(x_{t+1}, a)$
\STATE $m_i = 0, \quad i \in 0, \dots, N-1$
\FOR{$j \in 0, \dots, N-1$}
	\STATE \textcolor{gray}{\# Compute the projection of $\hat \cT z_j$ onto the support $\{ z_i \}$}
	\STATE $\hat \cT z_j \leftarrow \left [ r_t + \gamma_t z_j \right ]^{\Vmax}_{\Vmin}$
	\STATE $b_j \leftarrow (\hat \cT z_j - \Vmin) / \Delta z$ \; \textcolor{gray}{\# $b_j \in [0, N-1]$}
	\STATE $l \leftarrow \lfloor b_j \rfloor$, $u \leftarrow \lceil b_j \rceil$
	\STATE \textcolor{gray}{\# Distribute probability of $\hat \cT z_j$}
	\STATE $m_l \leftarrow m_l + p_j(x_{t+1}, a^*) (u - b_j)$
	\STATE $m_u \leftarrow m_u +  p_j(x_{t+1}, a^*) (b_j - l)$
\ENDFOR
\OUTPUT $-\sum_i m_i \log p_i(x_t, a_t)$ \; \textcolor{gray}{\# Cross-entropy loss}
\end{algorithmic}
\end{algorithm}

\begin{figure*}[htb]
\begin{center}
\includegraphics[width=.95\textwidth]{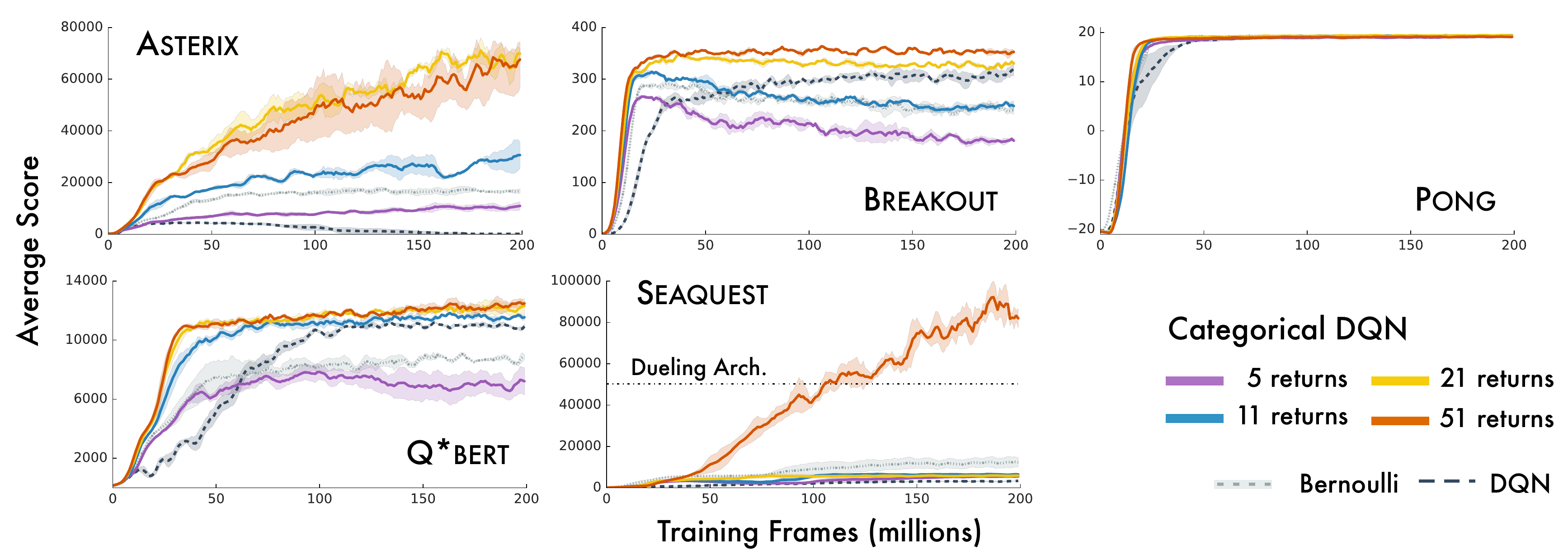}%
\caption{Categorical DQN: Varying number of atoms in the discrete distribution. Scores are moving averages over 5 million frames.}
\label{fig:atari_categorical}
\end{center}
\end{figure*}

\section{Evaluation on Atari 2600 Games}

To understand the approach in a complex setting, we applied the categorical algorithm to games from the Arcade Learning Environment \citep[ALE;][]{bellemare13arcade}. While the ALE is deterministic, stochasticity does occur in a number of guises: 1) from state aliasing, 2) learning from a nonstationary policy, and 3) from approximation errors.
We used five training games (Fig \ref{fig:atari_categorical}) and 52 testing games.

For our study, we use the DQN architecture \citep{mnih15nature}, but output the atom probabilities $p_i(x,a)$ instead of action-values, and chose $\Vmax = -\Vmin = 10$ from 
preliminary experiments over the training games. We call the resulting architecture \emph{Categorical DQN}.
We replace the squared loss $(r + \gamma Q(x',\pi(x')) - Q(x,a))^2$ by $\cL_{x,a}(\theta)$ and train the network to minimize this loss.\footnote{For $N=51$, our TensorFlow implementation trains at roughly 75\% of DQN's speed.} As in DQN, we use a simple $\epsilon$-greedy policy over the expected action-values; we leave as future work the many ways in which an agent could select actions on the basis of the full distribution. The rest of our training regime matches \citeauthor{mnih15nature}'s, including the use of a target network for $\tilde \theta$.

Figure \ref{fig:space_invaders_value_distribution} illustrates the typical value distributions we observed in our experiments. In this example, three actions (those including the button press) lead to the agent releasing its laser too early and eventually losing the game. The corresponding distributions reflect this: they assign a significant probability to 0 (the terminal value). The safe actions have similar distributions (\textsc{left}, which tracks the invaders' movement, is slightly favoured). This example helps explain
why our approach is so successful: the distributional update keeps separated the low-value, ``losing'' event from the high-value, ``survival'' event, rather than average them into one (unrealizable) expectation.\footnote{Video: \url{http://youtu.be/yFBwyPuO2Vg}.}

One surprising fact is that the distributions are not concentrated on one or two values, in spite of the ALE's determinism, but are often close to Gaussians. We believe this is due to our discretizing the diffusion process induced by $\gamma$.
\begin{figure}[htb]
\begin{center}
\includegraphics[width=3.2in]{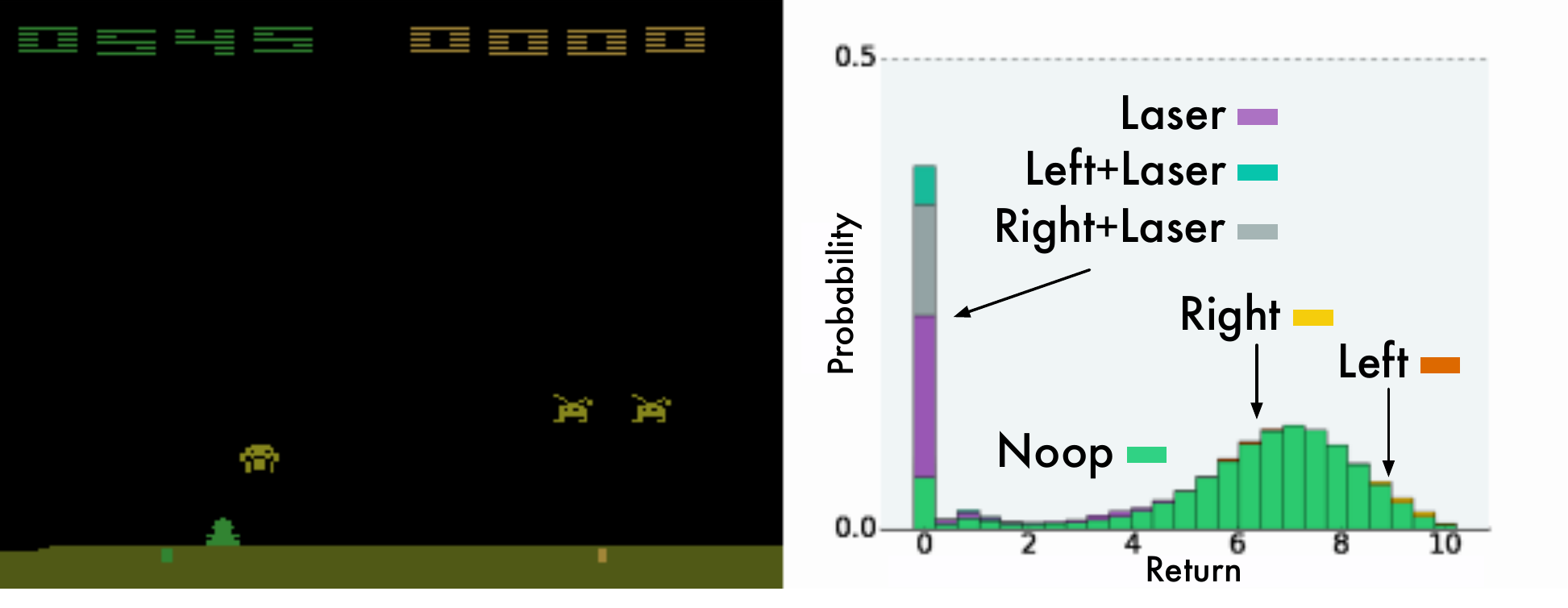}  %
\end{center}
\caption{Learned value distribution during an episode of \textsc{Space Invaders}. Different actions are shaded different colours. Returns below 0 (which do not occur in \textsc{Space Invaders}) are not shown here as the agent assigns virtually no probability to them.\label{fig:space_invaders_value_distribution}}
\end{figure}

\subsection{Varying the Number of Atoms}

We began by studying our algorithm's performance on the training games in relation to the number of atoms (Figure \ref{fig:atari_categorical}). For this experiment, we set $\epsilon = 0.05$. From the data, it is clear that using too few atoms can lead to poor behaviour, and that more always increases performance; this is not immediately obvious as we may have expected to saturate the network capacity. The difference in performance between the 51-atom version and DQN is particularly striking: the latter is outperformed in all five games, and in \textsc{Seaquest} we attain state-of-the-art performance. As an additional point of the comparison, the single-parameter Bernoulli algorithm performs better than DQN in 3 games out of 5, and is most notably more robust in \textsc{Asterix}.

One interesting outcome of this experiment was to find out that our method does pick up on stochasticity. \textsc{Pong} exhibits intrinsic randomness: the exact timing of the reward depends on internal registers and is truly unobservable. We see this clearly reflected in the agent's prediction (Figure \ref{fig:pong_bimodality}): over five consecutive frames, the value distribution shows two modes indicating the agent's belief that it has yet to receive a reward. Interestingly, since the agent's state does not include past rewards, it cannot even extinguish the prediction after receiving the reward, explaining the relative proportions of the modes.

\begin{figure*}[htb]
\begin{center}
\includegraphics[width=1.25in]{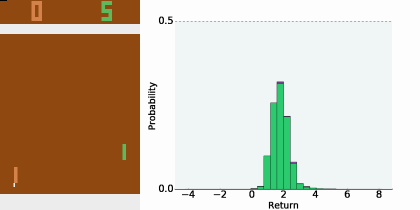}
\includegraphics[width=1.25in]{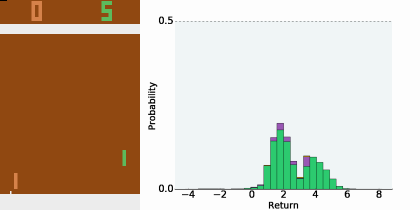}
\includegraphics[width=1.25in]{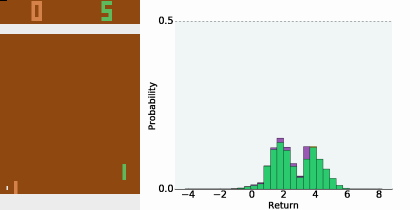}
\includegraphics[width=1.25in]{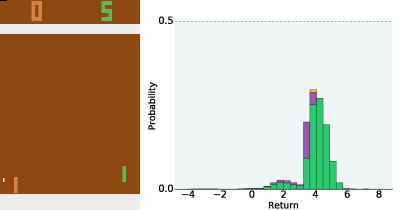}
\includegraphics[width=1.25in]{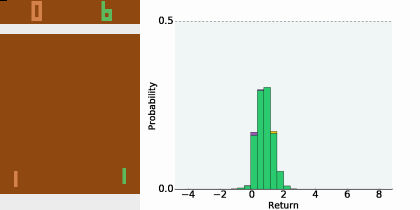}
\caption{Intrinsic stochasticity in \textsc{Pong}.\label{fig:pong_bimodality}}
\end{center}
\vspace{-1em}
\end{figure*}

\subsection{State-of-the-Art Results}

The performance of the 51-atom agent (from here onwards, C51) on the training games, presented in the last section, is particularly remarkable given that it involved none of the other algorithmic ideas present in state-of-the-art agents. We next asked whether incorporating the most common hyperparameter choice, namely a smaller training $\epsilon$, could lead to even better results. Specifically, we set $\epsilon = 0.01$ (instead of $0.05$); furthermore, every 1 million frames, we evaluate our agent's performance with $\epsilon = 0.001$. 

We compare our algorithm to DQN ($\epsilon = 0.01$), Double DQN \citep{vanhasselt16deep}, the Dueling architecture \citep{wang2016dueling}, and Prioritized Replay \citep{schaul16prioritized}, comparing the best evaluation score achieved during training.
We see that C51 significantly outperforms these other algorithms (Figures \ref{fig:perc_scores} and \ref{fig:allatari}). In fact, C51 surpasses the current state-of-the-art by a large margin in a number of games, most notably \textsc{Seaquest}. One particularly striking fact is the algorithm's good performance on sparse reward games, for example \textsc{Venture} and \textsc{Private Eye}. This suggests that value distributions are better able to propagate rarely occurring events. Full results are provided in the appendix.

We also include in the appendix (Figure \ref{fig:train_sup}) a comparison, averaged over 3 seeds, showing the number of games in which C51's training performance outperforms fully-trained DQN and human players. These results continue to show dramatic improvements, and are more representative of an agent's average performance. Within 50 million frames, C51 has outperformed a fully trained DQN agent on 45 out of 57 games. This suggests that the full 200 million training frames, and its ensuing computational cost, are unnecessary for evaluating reinforcement learning algorithms within the ALE.

The most recent version of the ALE contains a stochastic execution mechanism designed to ward against trajectory overfitting.%
Specifically, on each frame the environment rejects the agent's selected action with probability $p=0.25$.
Although DQN is mostly robust to stochastic execution, there are a few games in which its performance is reduced. On a score scale normalized with respect to the random and DQN agents, C51 obtains mean and median score improvements of $126\%$ and $21.5\%$ respectively, confirming the benefits of C51 beyond the deterministic setting.

\begin{figure}
\begin{tabular}{ l | r | r | r | r }
\multicolumn{1}{c}{} & \mbox{\textbf{Mean}} & \mbox{\textbf{Median}} & $>\, $\mbox{\textbf{H.B.}} & $>\,$\mbox{\textbf{DQN}} \\
\hline
\textsc{dqn}  &   228\% & 79\% & 24 & 0 \\
\textsc{ddqn}   &   307\% & 118\% & 33 & 43 \\
\textsc{Duel.}   &   373\% & 151\% & 37 & 50 \\
\textsc{Prior.}   &   434\% & 124\% & 39 & 48 \\
\textsc{Pr. Duel.}   &   592\% & 172\% & 39 & 44 \\
\hline
\textsc{C51}   &   \textbf{\textcolor{blue}{701\%}} & \textbf{\textcolor{blue}{178\%}} & \textbf{\textcolor{blue}{40}} & \textbf{\textcolor{blue}{50}} \\
\hline
\hline
\textsc{unreal}$^\dagger$ & 880\% & 250\% & - & - \\
\end{tabular}
\caption{Mean and median scores across 57 Atari games, measured as percentages of human baseline \citep[H.B.,][]{nair15massively}.}
\label{fig:perc_scores}
\end{figure}

\vspace{-5pt}
\begin{figure}[htb]
\begin{center}
\includegraphics[width=0.47\textwidth]{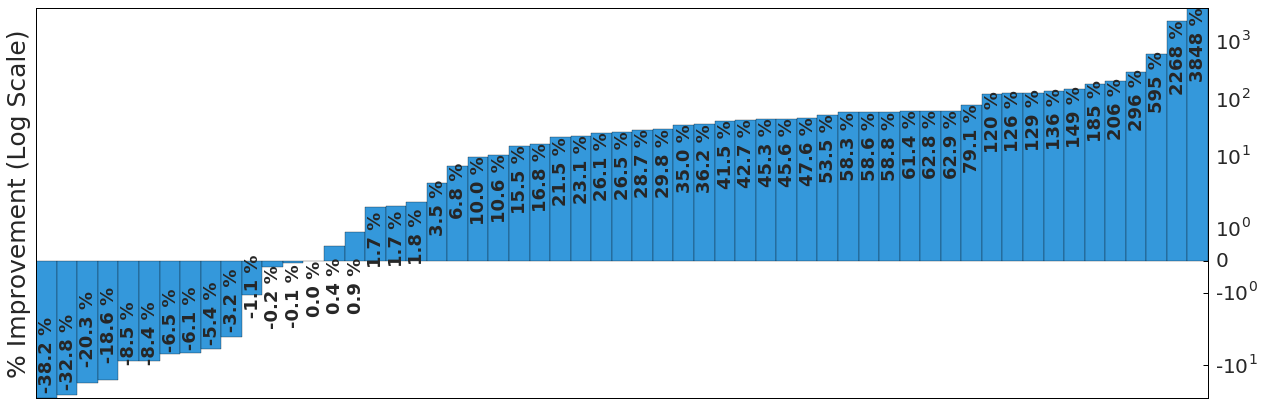}
\caption{Percentage improvement, per-game, of C51 over Double DQN, computed using \citeauthor{vanhasselt16deep}'s method.}
\label{fig:allatari}
\end{center}
\vspace{-1em}
\end{figure}

\section{Discussion}\label{sec:discussion}

In this work we sought a more complete picture of reinforcement learning, one that involves value distributions.
We found that learning value distributions is a powerful notion that allows us to surpass most gains previously made on Atari 2600, without further algorithmic adjustments.

\subsection{Why does learning a distribution matter?}
\blfootnote{$^\dagger$ The UNREAL results are not altogether comparable, as they were generated in the asynchronous setting with per-game hyperparameter tuning \citep{jaderberg17reinforcement}.}
It is surprising that, when we use a policy which aims to maximize expected return, we should see any difference in performance. The distinction we wish to make is that \emph{learning distributions matters in the presence of approximation}. We now outline some possible reasons.

\textbf{Reduced chattering.} Our results from Section \ref{sec:control} highlighted a significant instability in the Bellman optimality operator. When combined with function approximation, this instability may prevent the policy from converging, what \citet{gordon95stable} called \emph{chattering}. We believe the gradient-based categorical algorithm is able to mitigate these effects by effectively averaging the different distributions, similar to conservative policy iteration \citep{kakade02approximately}. While the chattering persists, it is integrated to the approximate solution.

\textbf{State aliasing.} Even in a deterministic environment, state aliasing may result in effective stochasticity.
\citet{mccallum95reinforcement}, for example, showed the importance of coupling representation learning with policy learning in partially observable domains. We saw an example of state aliasing in \textsc{Pong}, where the agent could not exactly predict the reward timing. Again, by explicitly modelling the resulting distribution we provide a more stable learning target.

\textbf{A richer set of predictions.} A recurring theme in artificial intelligence is the idea of an agent learning from a multitude of predictions (\citealt{caruana97multitask,utgoff02many,sutton11horde,jaderberg17reinforcement}). The distributional approach naturally provides us with a rich set of auxiliary predictions, namely: the probability that the return will take on a particular value. Unlike previously proposed approaches, however, the accuracy of these predictions is tightly coupled with the agent's performance.

\textbf{Framework for inductive bias.} The distributional perspective on reinforcement learning allows a more natural framework within which we can impose assumptions about the domain or the learning problem itself. In this work we used distributions with support bounded in $[\Vmin, \Vmax]$. Treating this support as a hyperparameter allows us to change the optimization problem by treating all extremal returns (e.g. greater than $\Vmax$) as equivalent.
Surprisingly, a similar value clipping in DQN significantly degrades performance in most games.
To take another example: interpreting the discount factor $\gamma$ as a proper probability, as some authors have argued, leads to a different algorithm.

\textbf{Well-behaved optimization.} It is well-accepted that the KL divergence between categorical distributions is a reasonably easy loss to minimize. This may explain some of our empirical performance. Yet early experiments with alternative losses, such as KL divergence between continuous densities, were not fruitful, in part because the KL divergence is insensitive to the values of its outcomes. A closer minimization of the Wasserstein metric should yield even better results than what we presented here.

In closing, we believe our results highlight the need to account for distribution in the design, theoretical or otherwise, of algorithms. 

\section*{Acknowledgements}

The authors acknowledge the important role played by their colleagues at DeepMind throughout the development of this work. Special thanks to Yee Whye Teh, Alex Graves, Joel Veness, Guillaume Desjardins, Tom Schaul, David Silver, Andre Barreto, Max Jaderberg, Mohammad Azar, Georg Ostrovski, Bernardo Avila Pires, Olivier Pietquin, Audrunas Gruslys, Tom Stepleton, Aaron van den Oord; and particularly Chris Maddison for his comprehensive review of an earlier draft. Thanks also to Marek Petrik for pointers to the relevant literature, and Mark Rowland for fine-tuning details in the final version.

\section*{Erratum}
The camera-ready copy of this paper incorrectly reported a mean score of 1010\% for C51. The corrected figure stands at 701\%, which remains higher than the other comparable baselines. The median score remains unchanged at 178\%. 

The error was due to evaluation episodes in one game (Atlantis) lasting over 30 minutes; in comparison, the other results presented here cap episodes at 30 minutes, as is standard. The previously reported score on Atlantis was 3.7 million; our 30-minute score is 841,075, which we believe is close to the achievable maximum in this time frame. Capping at 30 minutes brings our human-normalized score on Atlantis from 22824\% to a mere (!) 5199\%, unfortunately enough to noticeably affect the mean score, whose sensitivity to outliers is well-documented.

\begin{small}
\bibliography{value-distributions}
\bibliographystyle{icml2017}
\end{small}

\clearpage
\newpage

\appendix

\section{Related Work}

To the best of our knowledge, the work closest to ours are two papers \citep{morimura2010nonparametric,morimura10parametric} studying the distributional Bellman equation from the perspective of its cumulative distribution functions. The authors propose both parametric and nonparametric solutions to learn distributions for risk-sensitive reinforcement learning. They also provide some theoretical analysis for the policy evaluation setting, including a consistency result in the nonparametric case. By contrast, we also analyze the control setting, and emphasize the use of the distributional equations to improve approximate reinforcement learning.

The variance of the return has been extensively studied in the risk-sensitive setting. Of note, \citet{tamar2016learning} analyze the use of linear function approximation to learn this variance for policy evaluation, and \citet{prashanth13actorcritic} estimate the return variance in the design of a risk-sensitive actor-critic algorithm. \citet{mannor11meanvariance} provides negative results regarding the computation of a variance-constrained solution to the optimal control problem.

The distributional formulation also arises when modelling uncertainty. \citet{dearden98bayesian} considered a Gaussian approximation to the value distribution, and modelled the uncertainty over the parameters of this approximation using a Normal-Gamma prior. \citet{engel05reinforcement} leveraged the distributional Bellman equation to define a Gaussian process over the unknown value function. More recently, \citet{geist10kalman} proposed an alternative solution to the same problem based on unscented Kalman filters. We believe much of the analysis we provide here, which deals with the intrinsic randomness of the environment, can also be applied to modelling uncertainty.

Our work here is based on a number of foundational results, in particular concerning alternative optimality criteria. Early on, \citet{jaquette73markov} showed that a \emph{moment optimality} criterion, which imposes a total ordering on distributions, is achievable and defines a stationary optimal policy, echoing the second part of Theorem \ref{thm:control_convergence}. \citet{sobel82variance} is usually cited as the first reference to Bellman equations for the higher moments (but not the distribution) of the return. \citet{chung87discounted} provides results concerning the convergence of the distributional Bellman operator in total variation distance. \citet{white88mean} studies ``nonstandard MDP criteria'' from the perspective of optimizing the state-action pair occupancy.

A number of probabilistic frameworks for reinforcement learning have been proposed in recent years. The \emph{planning as inference} approach \citep{toussaint06probabilistic,hoffman09expectation} embeds the return into a graphical model, and applies probabilistic inference to determine the sequence of actions leading to maximal expected reward. \citet{wang08dual} considered the dual formulation of reinforcement learning, where one optimizes the stationary distribution subject to constraints given by the transition function \citep{puterman94markov}, in particular its relationship to linear approximation. Related to this dual is the Compress and Control algorithm \citet{veness15compress}, which describes a value function by learning a return distribution using density models. One of the aims of this work was to address the question left open by their work of whether one could be design a practical distributional algorithm based on the Bellman equation, rather than Monte Carlo estimation.

\section{Proofs}

\setcounter{lem}{0}
\setcounter{thm}{0}
\setcounter{prop}{0}

\begin{lem}[Partition lemma]
Let $A_1, A_2, \dots$ be a set of random variables describing a partition of $\Omega$, i.e. $A_i(\omega) \in \{0, 1\}$ and for any $\omega$ there is exactly one $A_i$ with $A_i(\omega) = 1$. Let $U, V$ be two random variables. Then
\begin{equation*}
d_p\big (U, V \big) \le \sum\nolimits_i d_p (A_i U, A_i V) .
\end{equation*}
\end{lem}
\begin{proof}
We will give the proof for $p < \infty$, noting that the same applies to $p = \infty$. 
Let $Y_i \Ddef A_i U$ and $Z_i \Ddef A_i V$, respectively. First note that
\begin{align}
d^p_p(A_i U, A_i V) &= \inf_{Y_i, Z_i} \expect \big [ \panorm{Y_i - Z_i} \big ] \nonumber \\
&= \inf_{Y_i, Z_i} \expect \Big [ \expect \big [ \panorm{Y_i - Z_i} \cbar A_i \big ] \Big ] \nonumber .
\end{align}
Now, $\panorm{A_i U - A_i V} = 0$ whenever $A_i = 0$. It follows that we can choose $Y_i, Z_i$ so that also $\panorm{Y_i - Z_i} = 0$ whenever $A_i = 0$, without increasing the expected norm. Hence
\begin{align}
d^p_p(A_i U, A_i V) &= \nonumber \\
&\hspace{-3.5em}\inf_{Y_i, Z_i} \Pr\{A_i = 1\} \expect \big [ \panorm{Y_i - Z_i} \cbar A_i = 1 \big ]. \label{eqn:norm_dp_ai}
\end{align}
Next, we claim that
\begin{align}
& \inf_{U, V} \sum\nolimits_i \Pr\{A_i = 1\} \expect \Big [ \big | A_i U - A_i V \big |^p \cbar A_i = 1 \Big ] \label{eqn:infimum_bound} \\
&\hspace{2em} \le \inf_{\substack{Y_1, Y_2, \dots\\Z_1, Z_2, \dots}} \sum\nolimits_i \Pr\{A_i = 1\} \expect \Big [ | Y_i - Z_i \big |^p \cbar A_i = 1 \Big ] . \nonumber
\end{align}
Specifically, the left-hand side of the equation is an infimum over all r.v.'s whose cumulative distributions are $F_U$ and $F_V$, respectively, while the right-hand side is an infimum over sequences of r.v's $Y_1, Y_2, \dots$ and $Z_1, Z_2, \dots$ whose cumulative distributions are $F_{A_i U}, F_{A_i V}$, respectively.
To prove this upper bound, consider the c.d.f. of $U$:
\begin{align*}
F_U(y) &= \Pr \{ U \le y \} \\
&= \sum\nolimits_i \Pr \{ A_i = 1 \} \Pr \{ U \le y \cbar A_i = 1 \} \\
&= \sum\nolimits_i \Pr \{ A_i = 1 \} \Pr \{ A_i U \le y \cbar A_i = 1 \} .
\end{align*}
Hence the distribution $F_U$ is equivalent, in an almost sure sense, to one that first picks an element $A_i$ of the partition, then picks a value for $U$ conditional on the choice $A_i$. On the other hand, the c.d.f. of $Y_i \Deq A_i U$ is
\begin{align*}
F_{A_i U}(y) &= \Pr \{ A_i = 1 \} \Pr \{ A_i U \le y \cbar A_i = 1 \} \\
&\hspace{1em} + \Pr \{ A_i = 0 \} \Pr \{ A_i U \le y \cbar A_i = 0 \} \\
&= \Pr \{ A_i = 1 \} \Pr \{ A_i U \le y \cbar A_i = 1 \} \\
&\hspace{1em} + \Pr \{ A_i = 0 \} \indic{y \ge 0} .
\end{align*}
Thus the right-hand side infimum in \eqnref{infimum_bound} has the additional constraint that it must preserve the conditional c.d.fs, in particular when $y \ge 0$.
Put another way, instead of having the freedom to completely reorder the mapping $U : \Omega \to \bR$, we can only reorder it within each element of the partition. We now write
\begin{align*}
d^p_p(U, V) &= \inf_{U,V} \pnorm{U - V} \\
&\hspace{-3.5em} = \inf_{U, V} \expect \big [ \panorm{U - V} \big ] \\
&\hspace{-3.5em} \overset{(a)}{=} \inf_{U, V} \sum\nolimits_i \Pr \{ A_i = 1 \} \expect \big [ \panorm{U - V} \cbar A_i = 1 \big ] \\
&\hspace{-3.5em} = \inf_{U, V} \sum\nolimits_i \Pr \{ A_i = 1 \} \expect \big [ \panorm{A_i U - A_i V} \cbar A_i = 1 \big ],
\end{align*}
where (a) follows because $A_1, A_2, \dots$ is a partition. Using \eqnref{infimum_bound}, this implies
\begin{align*}
d^p_p(U, V) & \\
&\hspace{-3.5em} = \inf_{U, V} \sum\nolimits_i \Pr\{A_i = 1\} \expect \Big [ \big | A_i U - A_i V \big |^p \cbar A_i = 1 \Big ] \\
&\hspace{-3.5em} \le \inf_{\substack{Y_1, Y_2, \dots\\Z_1, Z_2, \dots}} \sum\nolimits_i \Pr\{A_i = 1\} \expect \Big [ \big | Y_i - Z_i \big |^p \cbar A_i = 1 \Big ] \\
&\hspace{-3.5em} \overset{(b)}{=} \sum\nolimits_i \inf_{Y_i, Z_i} \Pr\{A_i = 1\} \expect \Big [ \big | Y_i - Z_i \big |^p \cbar A_i = 1 \Big ] \\
&\hspace{-3.5em} \overset{(c)}{=} \sum\nolimits_i d_p(A_i U, A_i V), 
\end{align*}
because in (b) the individual components of the sum are independently minimized; and (c) from \eqnref{norm_dp_ai}.
\end{proof}
\begin{lem}
$\dip$ is a metric over value distributions.
\end{lem}
\begin{proof}
The only nontrivial property is the triangle inequality. For any value distribution $Y \in \cZ$, write
\begin{align*}
\dip(Z_1, Z_2) &= \sup_{x,a} d_p(Z_1(x,a), Z_2(x,a)) \\
&\hspace{-4.25em} \overset{(a)}{\le}\sup_{x,a} \left [ d_p(Z_1(x,a), Y(x,a)) + d_p(Y(x,a), Z_2(x,a)) \right ]\\
&\hspace{-4.25em} \le \sup_{x,a} d_p(Z_1(x,a), Y(x,a)) + \sup_{x,a} d_p(Y(x,a), Z_2(x,a)) \\
&\hspace{-4.25em} = \dip(Z_1, Y) + \dip(Y, Z_2) ,
\end{align*}
where in (a) we used the triangle inequality for $d_p$.
\end{proof}

\begin{lem}
$\cTpi : \cZ \to \cZ$ is a $\gamma$-contraction in $\dip$.
\end{lem}
\begin{proof}
Consider $Z_1, Z_2 \in \cZ$. By definition,
\begin{equation}
\dip(\cTpi Z_1, \cTpi Z_2) = \sup_{x,a} d_p(\cTpi Z_1(x, a), \cTpi Z_2(x, a)) . \label{eqn:definition_tcpi_distance}
\end{equation}
By the properties of $d_p$, we have
\begin{small}
\begin{align*}
d_p(\cTpi Z_1(x,a), \cTpi Z_2(x,a)) &\\
&\hspace{-10em}= d_p(R(x,a) + \gamma P^\pi Z_1(x,a), R(x,a) + \gamma P^\pi Z_2(x,a)) \\
&\hspace{-10em}\le \gamma d_p(P^\pi Z_1(x,a), P^\pi Z_2(x,a)) \\
&\hspace{-10em}\le \gamma \sup_{x',a'} d_p(Z_1(x',a'), Z_2(x',a')),
\end{align*}
\end{small}
where the last line follows from the definition of $P^\pi$ (see \eqnref{policy_operator}). Combining with \eqnref{definition_tcpi_distance} we obtain
\begin{align*}
\hspace{-1.5em}\dip(\cTpi Z_1, \cTpi Z_2) &= \sup_{x,a} d_p(\cTpi Z_1(x,a), \cTpi Z_2(x,a)) \\
&\le \gamma \sup_{x', a'} d_p(Z_1(x',a'), Z_2(x', a')) \hspace{-2em} \\
&= \gamma \dip(Z_1, Z_2) .\qedhere
\end{align*}
\end{proof}

\begin{prop}[Sobel, 1982]
Consider two value distributions $Z_1, Z_2 \in \cZ$, and write $\Var(Z_i)$ to be the vector of variances of $Z_i$. Then
\begin{align*}
\infnorm{\expect \cTpi Z_1 - \expect \cTpi Z_2} &\le \gamma \infnorm{\expect Z_1 - \expect Z_2} \text{, and} \\
\infnorm{\Var(\cTpi Z_1) - \Var(\cTpi Z_2)} &\le \gamma^2 \infnorm{\Var Z_1 - \Var Z_2} .
\end{align*}
\begin{proof}
The first statement is standard, and its proof follows from $\expect \cTpi Z = \cTpi \expect Z$, where the second $\cTpi$ denotes the usual operator over value functions. Now, by independence of $R$ and $P^\pi Z_i$:
\begin{align*}
\Var(\cTpi Z_i(x,a)) &= \Var\Big(R(x,a) + \gamma P^\pi Z_i(x,a)\Big)  \\
&= \Var(R(x,a)) + \gamma^2 \Var(P^\pi Z_i(x,a)) .
\end{align*}
And now
\begin{align*}
\infnorm{\Var (\cTpi Z_1) - \Var(\cTpi Z_2)} & \\
&\hspace{-10em}= \sup_{x,a} \big | \Var(\cTpi Z_1(x,a)) - \Var(\cTpi Z_2(x,a)) \big | \\
&\hspace{-10em}= \sup_{x,a} \gamma^2 \big | \left [ \Var(P^\pi Z_1(x,a)) - \Var(P^\pi Z_2(x,a)) \right ] \big | \\
&\hspace{-10em}= \sup_{x,a} \gamma^2 \big | \expect \left [ \Var(Z_1(X', A')) - \Var(Z_2(X', A')) \right ] \big | \\
&\hspace{-10em}\le \sup_{x',a'} \gamma^2 \big | \Var(Z_1(x',a')) - \Var(Z_2(x', a')) \big | \\
&\hspace{-10em}\le \gamma^2 \infnorm{\Var Z_1 - \Var Z_2} . \qedhere
\end{align*}
\end{proof}
\end{prop}

\begin{lem}
Let $Z_1, Z_2 \in \cZ$. Then
\begin{equation*}
\infnorm{\expect \cT Z_1 - \expect \cT Z_2} \le \gamma \infnorm{\expect Z_1 - \expect Z_2},
\end{equation*}
and in particular $\expect Z_k \to Q^*$ exponentially quickly.
\end{lem}
\begin{proof}
The proof follows by linearity of expectation. Write $\cT_D$ for the distributional operator and $\cT_E$ for the usual operator. Then
\begin{align*}
\infnorm{\expect \cT_D Z_1 - \expect \cT_D Z_2} &= \infnorm{\cT_E \expect Z_1 - \cT_E \expect Z_2} \\
&\le \gamma \infnorm{Z_1 - Z_2}. \qedhere
\end{align*}
\end{proof}

\begin{thm}[Convergence in the control setting]
Let $Z_k := \cT Z_{k-1}$ with $Z_0 \in \cZ$. Let $\cX$ be measurable and suppose that $\cA$ is finite. Then
\begin{equation*}
\lim_{k \to \infty} \inf_{Z^{**} \in \cZ^{**}} d_p (Z_k(x,a), Z^{**}(x,a)) = 0 \quad \forall x, a .
\end{equation*}
If $\cX$ is finite, then $Z_k$ converges to $\cZ^{**}$ uniformly. Furthermore, if there is a total ordering $\prec$ on $\Pi^*$, such that for any $Z^* \in \cZ^*$,
\begin{equation*}
\cT Z^* = \cT^\pi Z^* \; \text{with} \; \pi \in \cG_{Z^*}, \; \pi \prec \pi' \; \; \forall \pi' \in \cG_{Z^*} \setminus \{\pi \},
\end{equation*}
then $\cT$ has a unique fixed point $Z^* \in \cZ^*$.%
\end{thm}

The gist of the proof of Theorem \ref{thm:control_convergence} consists in showing that for every state $x$, there is a time $k$ after which the greedy policy w.r.t. $Q_k$ is mostly optimal. 
To clearly expose the steps involved, we will first assume a unique (and therefore deterministic) optimal policy $\pi^*$, and later return to the general case; we will denote the optimal action at $x$ by $\pi^*(x)$.
For notational convenience, we will write $Q_k := \expect Z_k$ and $\cG_k := \cG_{Z_k}$.
Let $\Vrange := 2 \sup_{Z \in \cZ} \infnorm{Z} < \infty$ and let $\epsilon_k := \gamma^k \Vrange$. 
We first define the set of states $\cXk \subseteq \cX$ whose values must be sufficiently close to $Q^*$ at time $k$:
\begin{equation}
\cXk := \Big \{ x : Q^*(x, \pi^*(x)) - \max_{a \ne \pi^*(x)} Q^*(x,a) > 2 \epsilon_k \Big \} .\label{eqn:defn_cxk}
\end{equation}
Indeed, by Lemma \ref{lem:exponential_convergence_of_mean}, we know that after $k$ iterations
\begin{equation*}
| Q_k(x,a) - Q^*(x,a) | \le \gamma^k | Q_0(x,a) - Q^*(x,a) | \le \epsilon_k .
\end{equation*}
For $x \in \cX$, write $a^* := \pi^*(x)$. For any $a \in \cA$, we deduce that
\begin{equation*}
Q_k(x,a^*) - Q_k(x, a) \ge Q^*(x, a^*) - Q^*(x, a) - 2 \epsilon_k .
\end{equation*}
It follows that if $x \in \cX_k$, then also $Q_k(x, a^*) > Q_k(x, a')$ for all $a' \ne \pi^*(x)$: for these states, the greedy policy $\pi_k(x) := \argmax_a Q_k(x, a)$ corresponds to the optimal policy $\pi^*$.

\begin{lem}\label{lem:x_k_grows_monotonically}
For each $x \in \cX$ there exists a $k$ such that, for all $k' \ge k$, $x \in \cX_{k'}$, and in particular $\argmax_a Q_k(x,a) = \pi^*(x)$.
\end{lem}
\begin{proof}
Because $\cA$ is finite, the gap
\begin{equation*}
\Delta(x) := Q^*(x,\pi^*(x)) - \max_{a \ne \pi^*(x)} Q^*(x,a)
\end{equation*}
is attained for some strictly positive $\Delta(x) > 0$. By definition, there exists a $k$ such that
\begin{equation*}
\epsilon_k = \gamma^k B < \frac{\Delta(x)}{2},
\end{equation*}
and hence every $x \in \cX$ must eventually be in $\cX_k$.
\end{proof}
This lemma allows us to guarantee the existence of an iteration $k$ after which sufficiently many states are well-behaved, in the sense that the greedy policy at those states chooses the optimal action. We will call these states ``solved''. We in fact require not only these states to be solved, but also most of their successors, and most of the successors of those, and so on. We formalize this notion as follows: fix some $\delta > 0$, let $\cX_{k,0} := \cX_k$, and define for $i > 0$ the set
\begin{equation*}
\cX_{k,i} := \big \{ x : x \in \cXk, P(\cX_{k-1,i-1} \cbar x, \pi^*(x)) \ge 1 - \delta \big \},
\end{equation*}
As the following lemma shows, any $x$ is eventually contained in the recursively-defined sets $\cX_{k,i}$, for any $i$.

\begin{lem}\label{lem:x_k_i_grows_monotonically}
For any $i \in \bN$ and any $x \in \cX$, there exists a $k$ such that for all $k' \ge k$, $x \in \cX_{k',i}$.
\end{lem}
\begin{proof}
Fix $i$ and let us suppose that $\cX_{k,i} \uparrow \cX$. By Lemma \ref{lem:x_k_grows_monotonically}, this is true for $i=0$. We infer that for any probability measure $P$ on $\cX$, $P(\cX_{k,i}) \to P(\cX) = 1$. In particular, for a given $x \in \cX_k$, this implies that
\begin{equation*}
P(\cX_{k, i} \cbar x, \pi^*(x)) \to P(\cX \cbar x, \pi^*(x)) = 1.
\end{equation*}
Therefore, for any $x$, there exists a time after which it is and remains a member of $\cX_{k,i+1}$, the set of states for which $P(\cX_{k-1, i} \cbar x, \pi^*(x)) \ge 1 - \delta$. We conclude that $\cX_{k,i+1} \uparrow \cX$ also. The statement follows by induction.
\end{proof}

\begin{proof}[Proof of Theorem \ref{thm:control_convergence}]
The proof is similar to policy iteration-type results, but requires more care in dealing with the metric and the possibly infinite state space.
We will write $W_k(x) := Z_k(x, \pi_k(x))$, define $W^*$ similarly and with some overload of notation write $\cT W_k(x) := W_{k+1}(x) = \cT Z_k(x, \pi_{k+1}(x))$.
Finally, let $S^k_i(x) := \indic{x \in \cX_{k,i}}$ and $\bar S^k_i(x) = 1 - S^k_i(x)$. 

Fix $i > 0$ and $x \in \cX_{k+1,i+1} \subseteq \cX_k$. We begin by using Lemma \ref{lem:partition_lemma} to separate the transition from $x$ into a solved term and an unsolved term:
\begin{equation*}
P^{\pi_k} W_k(x) = S^k_i W_k(X') + \bar S^k_i W_k(X'),
\end{equation*}
where $X'$ is the random successor from taking action $\pi_k(x) := \pi^*(x)$, and we write $S^k_i = S^k_i(X'), \bar S^k_i = \bar S^k_i(X')$ to ease the notation. Similarly,
\begin{equation*}
P^{\pi_k} W^*(x) = S^k_i W^*(X') + \bar S^k_i W^*(X') .
\end{equation*}
Now
\begin{align}
d_p(W_{k+1}(x), W^*(x)) &= d_p(\cT W_k (x), \cT W^*(x)) \nonumber \\
&\hspace{-7em} \overset{(a)}{\le} \gamma d_p(P^{\pi_k} W_k(x), P^{\pi^*} W^*(x)) \nonumber \\  %
&\hspace{-7em} \overset{(b)}{\le} \gamma d_p( S^k_i W_k(X'), S^k_i W^*(X')) \nonumber \\
&\hspace{-6em} + \gamma d_p (\bar S^k_i W_k(X'), \bar S^k_i W^*(X')), \label{eqn:partition_dp}
\end{align}
where in $(a)$ we used Properties P1 and P2 of the Wasserstein metric, and in (b) we separate states for which $\pi_k = \pi^*$ from the rest using Lemma \ref{lem:partition_lemma} ($\{S^k_i, \bar S^k_i \}$ form a partition of $\Omega$).
Let $\delta_i := \Pr \{ X' \notin \cX_{k,i} \} = \expect \{ \bar S^k_i(X') \} = \pnorm{\bar S_i^k (X')}$.
From property P3 of the Wasserstein metric, we have
\begin{align*}
d_p (\bar S^k_i W_k(X'), \bar S^k_i W^*(X')) &\\
&\hspace{-7em} \le \sup_{x'} d_p(\bar S^k_i (X') W_k(x'), \bar S^k_i (X') W^*(x')) \\
&\hspace{-7em} \le \pnorm{\bar S_i^k(X')} \sup_{x'} d_p(W_k(x'), W^*(x')) \\
&\hspace{-7em} \le \delta_i \sup_{x'} d_p(W_k(x'), W^*(x')) \\
&\hspace{-7em} \le \delta_i B .
\end{align*}
Recall that $B < \infty$ is the largest attainable $\infnorm{Z}$. Since also $\delta_i < \delta$ by our choice of $x \in \cX_{k+1, i+1}$, we can upper bound the second term in \eqnref{partition_dp} by $\gamma \delta B$. This yields
\begin{align*}
d_p(W_{k+1}(x), W^*(x)) &\le \\
&\hspace{-6em} \gamma d_p(S^k_i W_k(X'), S^k_i W^*(X')) + \gamma \delta B .
\end{align*}
By induction on $i > 0$, we conclude that for $x \in \cX_{k+i, i}$ and some random state $X''$ $i$ steps forward,
\begin{align*}
d_p(W_{k+i}(x), W^*(x)) &\le \\
&\hspace{-6em} \gamma^i d_p(S^k_0 W_k(X''), S^k_0 W^*(X'')) + \frac{\delta B}{1 - \gamma} \\
&\hspace{-6em} \le \gamma^i B + \frac{\delta B}{1 - \gamma} .
\end{align*}
Hence for any $x \in \cX$, $\epsilon > 0$, we can take $\delta$, $i$, and finally $k$ large enough to make $d_p(W_k(x), W^*(x)) < \epsilon$. The proof then extends to $Z_k(x,a)$ by considering one additional application of $\cT$.

We now consider the more general case where there are multiple optimal policies. We expand the definition of $\cX_{k,i}$ as follows:
\begin{small}
\begin{equation*}
\cX_{k,i} := \big \{ x \in \cXk :\forall \pi^* \in \Pi^*, \; \;\mathclap{\expect_{a^* \sim \pi^*(x)}} \; \; P(\cX_{k-1,i-1} \cbar x, a^*) \ge 1 - \delta  \big \},
\end{equation*}
\end{small}
Because there are finitely many actions, Lemma \ref{lem:x_k_i_grows_monotonically} also holds for this new definition. As before, take $x \in \cX_{k,i}$, but now consider the sequence of greedy policies $\pi_{k}, \pi_{k-1}, \dots$ selected by successive applications of $\cT$, and write
\begin{equation*}
\cT^{\bar \pi_k} := \cT^{\pi_k} \cT^{\pi_{k-1}} \cdots \cT^{\pi_{k-i+1}},
\end{equation*}
such that
\begin{equation*}
Z_{k+1} = \cT^{\bar \pi_k} Z_{k-i+1} .
\end{equation*}
Now denote by $\cZ^{**}$ the set of nonstationary optimal policies. If we take any $Z^* \in \cZ^*$, we deduce that
\begin{equation*}
\inf_{Z^{**} \in \cZ^{**}} d_p(\cT^{\bar \pi_k} Z^*(x,a), Z^{**}(x,a)) \le \frac{\delta B}{1 - \gamma},
\end{equation*}
since $Z^*$ corresponds to some optimal policy $\pi^*$ and $\bar \pi_k$ is optimal along most of the trajectories from $(x,a)$. In effect, $\cT^{\bar \pi_k} Z^*$ is close to the value distribution of the nonstationary optimal policy $\bar \pi_k \pi^*$. Now for this $Z^*$,
\begin{align*}
\inf_{Z^{**}} d_p(Z_k(x,a), Z^{**}(x,a)) &\\
&\hspace{-8em}\le d_p(Z_k(x,a), \cT^{\bar \pi_k} Z^{*}(x,a)) \\
&\hspace{-7em}+ \inf_{Z^{**}} d_p(\cT^{\bar \pi_k} Z^*(x,a), Z^{**}(x,a)) \\
&\hspace{-8em}\le d_p(\cT^{\bar \pi_k} Z_{k-i+1}(x,a), \cT^{\bar \pi_k} Z^{*}(x,a)) + \frac{\delta B}{1 - \gamma} \\
&\hspace{-8em}\le \gamma^i B + \frac{2 \delta B}{1-\gamma},
\end{align*}
using the same argument as before with the newly-defined $\cX_{k,i}$.
It follows that
\begin{equation*}
\inf_{Z^{**} \in \cZ^{**}} d_p(Z_k(x,a), Z^{**}(x,a)) \to 0.
\end{equation*}
When $\cX$ is finite, there exists a fixed $k$ after which $\cX_k = \cX$. The uniform convergence result then follows.

To prove the uniqueness of the fixed point $Z^*$ when $\cT$ selects its actions according to the ordering $\prec$, we note that
for any optimal value distribution $Z^*$, its set of greedy policies is $\Pi^*$. Denote by $\pi^*$ the policy coming first in the ordering over $\Pi^*$. Then $\cT = \cT^{\pi^*}$, which has a unique fixed point (Section \ref{sec:policy_evaluation}).
\end{proof}

\setcounter{prop}{3}

\begin{prop}
That $\cT$ has a fixed point $Z^* = \cT Z^*$ is insufficient to guarantee the convergence of $\{ Z_k \}$ to $\cZ^*$.
\end{prop}
We provide here a sketch of the result. Consider a single state $x_1$ with two actions, $a_1$ and $a_2$ (Figure \ref{fig:nonstationary_example}). The first action yields a reward of $1/2$, while the other either yields $0$ or $1$ with equal probability, and both actions are optimal. Now take $\gamma = 1/2$ and write $R_0, R_1, \dots$ for the received rewards. Consider a stochastic policy that takes action $a_2$ with probability $p$. For $p = 0$, the return is
\begin{equation*}
Z_{p=0} = \frac{1}{1 - \gamma} \frac{1}{2} = 1.
\end{equation*}
For $p = 1$, on the other hand, the return is random and is given by the following fractional number (in binary):
\begin{equation*}
Z_{p=1} = R_0 . R_1 R_2 R_3 \cdots .
\end{equation*}
As a result, $Z_{p=1}$ is uniformly distributed between $0$ and $2$! In fact, note that
\begin{equation*}
Z_{p=0} = 0 . 1 1 1 1 1 \cdots = 1.
\end{equation*}
For some intermediary value of $p$, we obtain a different probability of the different digits, but always putting some probability mass on all returns in $[0, 2]$.

\begin{figure}
\begin{center}
\includegraphics[width=1.4in]{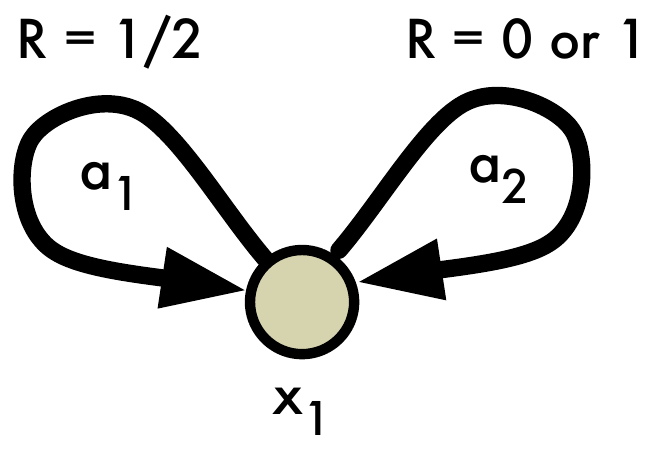}
\end{center}
\caption{A simple example illustrating the effect of a nonstationary policy on the value distribution.\label{fig:nonstationary_example}}
\end{figure}

Now suppose we follow the nonstationary policy that takes $a_1$ on the first step, then $a_2$ from there on. By inspection, the return will be uniformly distributed on the interval $[1/2, 3/2]$, which does not correspond to the return under any value of $p$.
But now we may imagine an operator $\cT$ which alternates between $a_1$ and $a_2$ depending on the exact value distribution it is applied to, which would in turn converge to a nonstationary optimal value distribution.

\begin{lem}[Sample Wasserstein distance]\label{prop:wasserstein_sgd}
Let $\{ P_i \}$ be a collection of random variables, $I \in \bN$ a random index independent from $\{ P_i \}$, and consider the mixture random variable $P = P_I$. For any random variable $Q$ independent of $I$,
\begin{equation*}
d_p(P, Q) \le \expect_{i \sim I} d_p(P_i, Q),
\end{equation*}
and in general the inequality is strict and
\begin{equation*}
\nabla_Q d_p(P_I, Q) \ne \expect_{i \sim I} \nabla_Q d_p(P_i, Q) .
\end{equation*}
\end{lem}
\begin{proof}
We prove this using Lemma \ref{lem:partition_lemma}. Let $A_i := \indic{I = i}$. We write
\begin{align*}
d_p(P, Q) &= d_p(P_I, Q) \\
&= d_p \Big (\sum\nolimits_i A_i P_i, \sum\nolimits_i A_i Q \Big) \\
&\le \sum\nolimits_i d_p(A_i P_i, A_i Q) \\
&\le \sum\nolimits_i \Pr\{I = i\} d_p(P_i, Q) \\
&= \expect\nolimits_I d_P(P_i, Q) .
\end{align*}
where in the penultimate line we used the independence of $I$ from $P_i$ and $Q$ to appeal to property P3 of the Wasserstein metric.

To show that the bound is in general strict, consider the mixture distribution depicted in Figure \ref{fig:wasserstein_counterexample}. We will simply consider the $d_1$ metric between this distribution $P$ and another distribution $Q$. The first distribution is
\begin{equation*}
P = \left \{ \begin{array}{ll}
	0 & \text{w.p. } 1/2 \\
	1 & \text{w.p. } 1/2 .
	\end{array} \right .
\end{equation*}
In this example, $i \in \{1, 2\}$, $P_1 = 0$, and $P_2 = 1$.
Now consider the distribution with the same support but that puts probability $p$ on $0$:
\begin{equation*}
Q = \left \{ \begin{array}{ll}
	0 & \text{w.p. } p \\
	1 & \text{w.p. } 1 - p .
	\end{array} \right .
\end{equation*}
The distance between $P$ and $Q$ is
\begin{equation*}
d_1(P, Q) = |p - \tfrac{1}{2}| .
\end{equation*}
This is $d_1(P, Q) = \frac{1}{2}$ for $p \in \{0, 1\}$, and strictly less than $\frac{1}{2}$ for any other values of $p$. On the other hand, the corresponding expected distance (after sampling an outcome $x_1$ or $x_2$ with equal probability) is
\begin{equation*}
\expect\nolimits_I d_1(P_i, Q) = \tfrac{1}{2} p + \tfrac{1}{2} (1 - p) = \tfrac{1}{2} .
\end{equation*}
Hence $d_1(P, Q) < \expect\nolimits_I d_1(P_i, Q)$
for $p \in (0, 1)$. This shows that the bound is in general strict. By inspection, it is clear that the two gradients are different.
\end{proof}
\begin{figure}[ht]
\begin{center}
\includegraphics[width=1.4in]{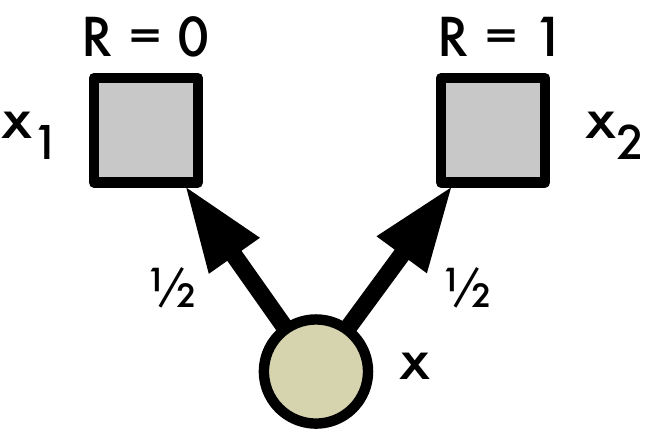}
\end{center}
\caption{Example MDP in which the expected sample Wasserstein distance is greater than the Wasserstein distance.\label{fig:wasserstein_counterexample}}
\end{figure}

\begin{prop}\label{prop:wasserstein_sgd_rl}
Fix some next-state distribution $Z$ and policy $\pi$. Consider a parametric value distribution $Z_\theta$, and and define the Wasserstein loss
\begin{equation*}
\cL_W(\theta) := d_p(Z_\theta(x,a), R(x,a) + \gamma Z(X', \pi(X'))) .
\end{equation*}
Let $r \sim R(x,a)$ and $x' \sim P(\cdot \cbar x,a)$ and consider the sample loss
\begin{equation*}
L_W(\theta, r, x') := d_p(Z_\theta(x, a), r + \gamma Z(x', \pi(x')) .
\end{equation*}
Its expectation is an upper bound on the loss $\cL_W$:
\begin{equation*}
\cL_W(\theta) \le \expect_{R, P} L_W(\theta, r, x'),
\end{equation*}
in general with strict inequality.
\end{prop}
The result follows directly from the previous lemma.

\section{Algorithmic Details}

\begin{figure*}
\begin{center}
\includegraphics[width=\textwidth]{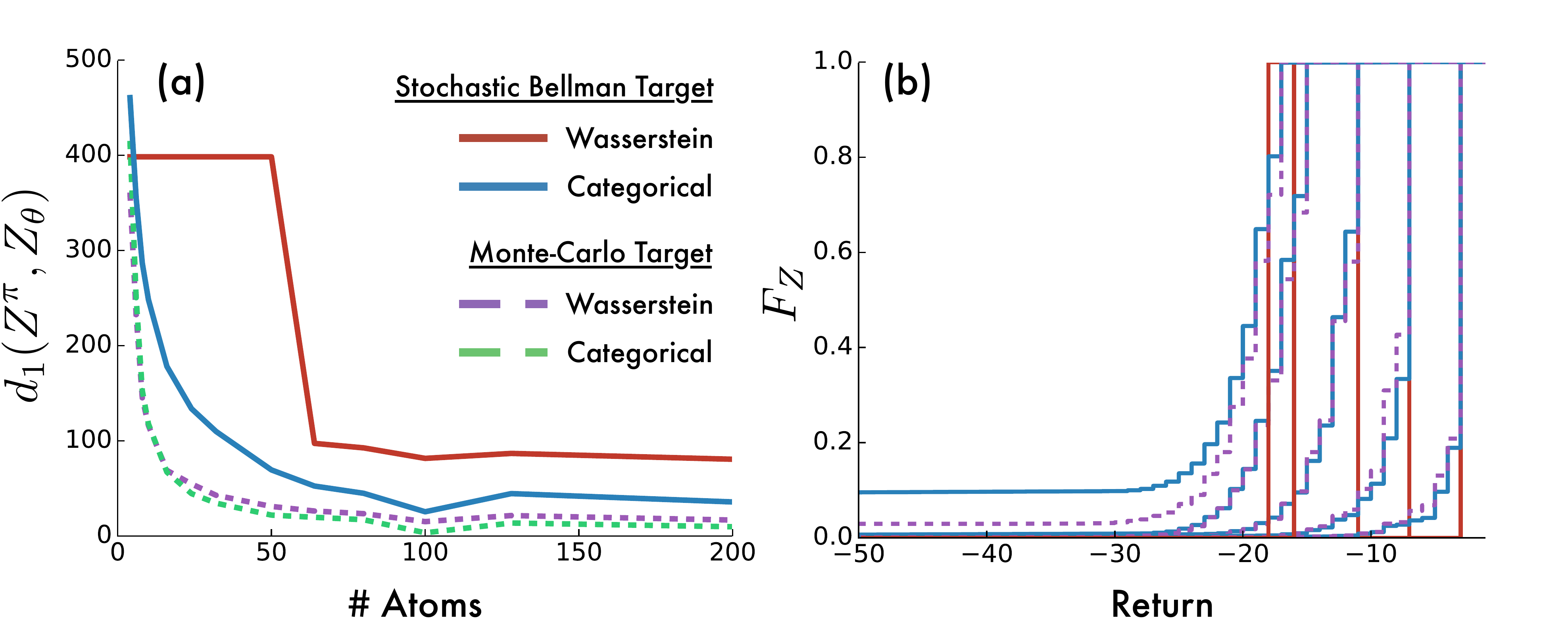}
\caption{(a) Wasserstein distance between ground truth distribution $Z^{\pi}$ and approximating distributions $Z_{\theta}$. Varying number of atoms in approximation, training target, and loss function. (b) Approximate cumulative distributions for five representative states in CliffWalk.}
\label{fig:cw_errors}
\end{center}
\end{figure*}

While our training regime closely follows that of DQN \cite{mnih15nature}, we use Adam \cite{kingma2014adam} instead of RMSProp \cite{tieleman2012lecture} for gradient rescaling. We also performed some hyperparameter tuning for our final results. Specifically, we evaluated two hyperparameters over our five training games and choose the values that performed best. The hyperparameter values we considered were $\Vmax \in \{3, 10, 100\}$ and $\epsilon_{adam} \in \{  1/L, 0.1/L, 0.01/L, 0.001/L,  0.0001/L\}$, where $L=32$ is the minibatch size. We found $\Vmax = 10$ and $\epsilon_{adam} = 0.01/L$ performed best. We used the same step-size value as DQN ($\alpha = 0.00025$).

Pseudo-code for the categorical algorithm is given in Algorithm \ref{alg:cql}. We apply the Bellman update to each atom separately, and then project it into the two nearest atoms in the original support. Transitions to a terminal state are handled with $\gamma_t = 0$.

\section{Comparison of Sampled Wasserstein Loss and Categorical Projection}

\begin{figure}[htb]
\begin{center}
\includegraphics[width=0.45\textwidth]{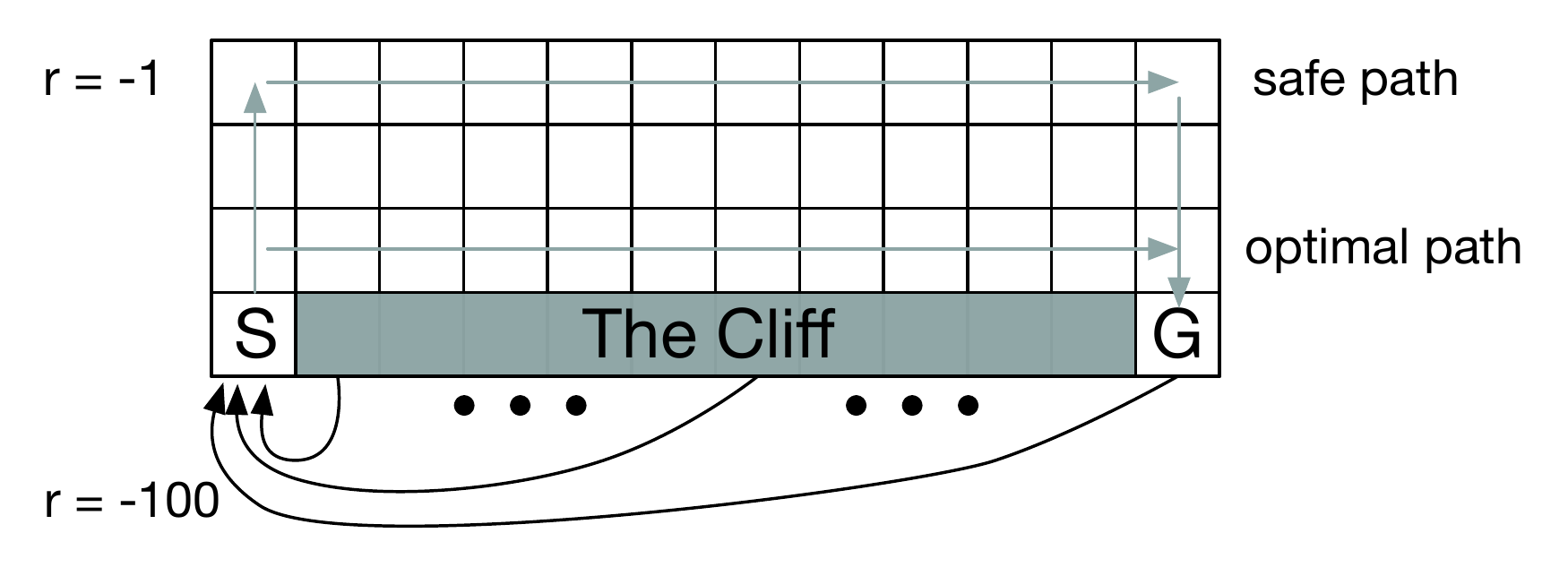}
\caption{CliffWalk Environment \cite{sutton98reinforcement}.}
\label{fig:cliffwalk}
\end{center}
\end{figure}

Lemma~\ref{lem:contraction_pe} proves that for a fixed policy $\pi$ the distributional Bellman operator is a $\gamma$-contraction in $\dip$, and therefore that $\cT^{\pi}$ will converge in distribution to the true distribution of returns $Z^{\pi}$.
In this section, we empirically validate these results on the CliffWalk domain shown in Figure~\ref{fig:cliffwalk}. The dynamics of the problem match those given by \citet{sutton98reinforcement}. We also study the convergence of the distributional Bellman operator under the sampled Wasserstein loss and the categorical projection (Equation~\ref{eqn:cat_proj}) while following a policy that tries to take the safe path but has a 10\% chance of taking another action uniformly at random.

We compute a ground-truth distribution of returns $Z^{\pi}$ using $10000$ Monte-Carlo (MC) rollouts from each state. We then perform two experiments, approximating the value distribution at each state with our discrete distributions.

In the first experiment, we perform supervised learning using either the Wasserstein loss or categorical projection (Equation~\ref{eqn:cat_proj}) with cross-entropy loss. We use $Z^{\pi}$ as the supervised target and perform $5000$ sweeps over all states to ensure both approaches have converged. In the second experiment, we use the same loss functions, but the training target comes from the one-step distributional Bellman operator with sampled transitions. We use $\Vmin = -100$ and $\Vmax = -1$.\footnote{Because there is a small probability of larger negative returns, some approximation error is unavoidable. However, this effect is relatively negligible in our experiments.} For the sample updates we perform 10 times as many sweeps over the state space. Fundamentally, these experiments investigate how well the two training regimes (minimizing the Wasserstein or categorical loss) minimize the Wasserstein metric under both ideal (supervised target) and practical (sampled one-step Bellman target) conditions.

In Figure~\ref{fig:cw_errors}a we show the final Wasserstein distance $d_1(Z^{\pi}, Z_\theta)$ between the learned distributions and the ground-truth distribution as we vary the number of atoms. The graph shows that the categorical algorithm does indeed minimize the Wasserstein metric in both the supervised and sample Bellman setting. It also highlights that minimizing the Wasserstein loss with stochastic gradient descent is in general flawed, confirming the intuition given by Proposition \ref{prop:wasserstein_sgd_rl}. In repeat experiments the process converged to different values of $d_1(Z^\pi, Z_\theta)$, suggesting the presence of local minima (more prevalent with fewer atoms).

Figure~\ref{fig:cw_errors} provides additional insight into why the sampled Wasserstein distance may perform poorly.
Here, we see the cumulative densities for the approximations learned under these two losses for five different states along the safe path in CliffWalk. The Wasserstein has converged to a fixed-point distribution, but not one that captures the true (Monte Carlo) distribution very well. By comparison, the categorical algorithm captures the variance of the true distribution much more accurately.

\section{Supplemental Videos and Results}

In Figure \ref{fig:videos} we provide links to supplemental videos showing the C51 agent during training on various Atari 2600 games. Figure \ref{fig:train_sup} shows the relative performance of C51 over the course of training. Figure \ref{fig:atari_sota} provides a table of evaluation results, comparing C51 to other state-of-the-art agents. Figures \ref{fig:atari_frames0}--\ref{fig:atari_frames4} depict particularly interesting frames.

\begin{figure}[htb]
\begin{center}
\includegraphics[width=0.45\textwidth]{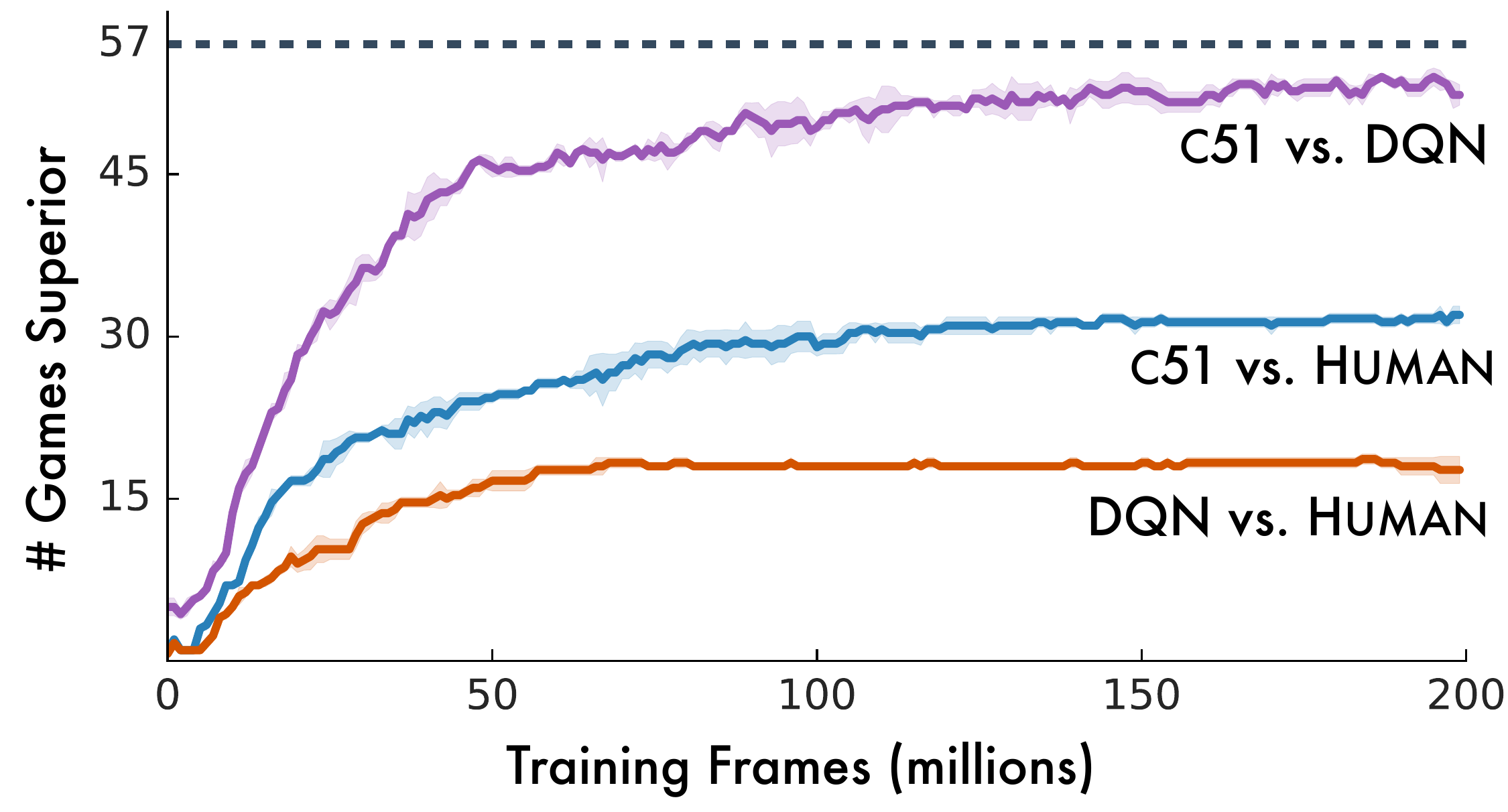}
\caption{Number of Atari games where an agent's training performance is greater than a baseline (fully trained DQN \& human). Error bands give standard deviations, and averages are over number of games.}
\label{fig:train_sup}
\end{center}
\end{figure}

\begin{figure}[htb]
\begin{tabular}{ l | r}
	\textbf{\textsc{games}}	&	\textbf{\textsc{video}} \textbf{\textsc{url}}\\
	\hline
	Freeway	& \url{http://youtu.be/97578n9kFIk} \\
	Pong		& \url{http://youtu.be/vIz5P6s80qA} \\
	Q*Bert	& \url{http://youtu.be/v-RbNX4uETw} \\
	Seaquest & \url{http://youtu.be/d1yz4PNFUjI} \\
	Space Invaders & \url{http://youtu.be/yFBwyPuO2Vg}
\end{tabular}
\caption{Supplemental videos of C51 during training.\label{fig:videos}}
\end{figure}

\newpage

\begin{figure*}
\small
\centering
\begin{tabular}{ l | r|r|r|r|r|r| r }
  \textbf{\textsc{games}}  &  \textbf{\textsc{random}}  &  \textbf{\textsc{human}}  &  \textbf{\textsc{dqn}}  &  \textbf{\textsc{ddqn}}  &  \textbf{\textsc{duel}}  &  \textbf{\textsc{prior.}} \textbf{\textsc{duel.}}  &  \textbf{\textsc{c51}}\\
\hline
  Alien  &   227.8  &  \textbf{\textcolor{blue}{ 7,127.7}}  &   1,620.0  &   3,747.7  &   4,461.4  &   3,941.0  &  3,166\\
  Amidar  &   5.8  &   1,719.5  &   978.0  &   1,793.3  &  \textbf{\textcolor{blue}{ 2,354.5}}  &   2,296.8  &  1,735\\
  Assault  &   222.4  &   742.0  &   4,280.4  &   5,393.2  &   4,621.0  &  \textbf{\textcolor{blue}{ 11,477.0}}  &  7,203\\
  Asterix  &   210.0  &   8,503.3  &   4,359.0  &   17,356.5  &   28,188.0  &   375,080.0  &  \textbf{\textcolor{blue}{406,211}}\\
  Asteroids  &   719.1  &  \textbf{\textcolor{blue}{ 47,388.7}}  &   1,364.5  &   734.7  &   2,837.7  &   1,192.7  &  1,516\\
  Atlantis  &   12,850.0  &   29,028.1  &   279,987.0  &   106,056.0  &   382,572.0  &   395,762.0  &  \textbf{\textcolor{blue}{841,075}}\\
  Bank Heist  &   14.2  &   753.1  &   455.0  &   1,030.6  &  \textbf{\textcolor{blue}{ 1,611.9}}  &   1,503.1  &  976\\
  Battle Zone  &   2,360.0  &  \textbf{\textcolor{blue}{ 37,187.5}}  &   29,900.0  &   31,700.0  &   37,150.0  &   35,520.0  &  28,742\\
  Beam Rider  &   363.9  &   16,926.5  &   8,627.5  &   13,772.8  &   12,164.0  &  \textbf{\textcolor{blue}{ 30,276.5}}  &  14,074\\
  Berzerk  &   123.7  &   2,630.4  &   585.6  &   1,225.4  &   1,472.6  &  \textbf{\textcolor{blue}{ 3,409.0}}  &  1,645\\
  Bowling  &   23.1  &  \textbf{\textcolor{blue}{ 160.7}}  &   50.4  &   68.1  &   65.5  &   46.7  &  81.8\\
  Boxing  &   0.1  &   12.1  &   88.0  &   91.6  &  \textbf{\textcolor{blue}{ 99.4}}  &   98.9  &  97.8\\
  Breakout  &   1.7  &   30.5  &   385.5  &   418.5  &   345.3  &   366.0  &  \textbf{\textcolor{blue}{748}}\\
  Centipede  &   2,090.9  &  \textbf{\textcolor{blue}{ 12,017.0}}  &   4,657.7  &   5,409.4  &   7,561.4  &   7,687.5  &  9,646\\
  Chopper Command  &   811.0  &   7,387.8  &   6,126.0  &   5,809.0  &   11,215.0  &   13,185.0  &  \textbf{\textcolor{blue}{15,600}}\\
  Crazy Climber  &   10,780.5  &   35,829.4  &   110,763.0  &   117,282.0  &   143,570.0  &   162,224.0  &  \textbf{\textcolor{blue}{179,877}}\\
  Defender  &   2,874.5  &   18,688.9  &   23,633.0  &   35,338.5  &   42,214.0  &   41,324.5  &  \textbf{\textcolor{blue}{47,092}}\\
  Demon Attack  &   152.1  &   1,971.0  &   12,149.4  &   58,044.2  &   60,813.3  &   72,878.6  &  \textbf{\textcolor{blue}{130,955}}\\
  Double Dunk  &   -18.6  &   -16.4  &   -6.6  &   -5.5  &   0.1  &   -12.5  &  \textbf{\textcolor{blue}{2.5}}\\
  Enduro  &   0.0  &   860.5  &   729.0  &   1,211.8  &   2,258.2  &   2,306.4  &  \textbf{\textcolor{blue}{3,454}}\\
  Fishing Derby  &   -91.7  &   -38.7  &   -4.9  &   15.5  &  \textbf{\textcolor{blue}{ 46.4}}  &   41.3  &  8.9\\
  Freeway  &   0.0  &   29.6  &   30.8  &   33.3  &   0.0  &   33.0  &  \textbf{\textcolor{blue}{33.9}}\\
  Frostbite  &   65.2  &   4,334.7  &   797.4  &   1,683.3  &   4,672.8  &  \textbf{\textcolor{blue}{ 7,413.0}}  &  3,965\\
  Gopher  &   257.6  &   2,412.5  &   8,777.4  &   14,840.8  &   15,718.4  &  \textbf{\textcolor{blue}{ 104,368.2}}  &  33,641\\
  Gravitar  &   173.0  &  \textbf{\textcolor{blue}{ 3,351.4}}  &   473.0  &   412.0  &   588.0  &   238.0  &  440\\
  H.E.R.O.  &   1,027.0  &   30,826.4  &   20,437.8  &   20,130.2  &   20,818.2  &   21,036.5  &  \textbf{\textcolor{blue}{38,874}}\\
  Ice Hockey  &   -11.2  &  \textbf{\textcolor{blue}{ 0.9}}  &   -1.9  &   -2.7  &   0.5  &   -0.4  &  -3.5\\
  James Bond  &   29.0  &   302.8  &   768.5  &   1,358.0  &   1,312.5  &   812.0  &  \textbf{\textcolor{blue}{1,909}}\\
  Kangaroo  &   52.0  &   3,035.0  &   7,259.0  &   12,992.0  &  \textbf{\textcolor{blue}{ 14,854.0}}  &   1,792.0  &  12,853\\
  Krull  &   1,598.0  &   2,665.5  &   8,422.3  &   7,920.5  &  \textbf{\textcolor{blue}{ 11,451.9}}  &   10,374.4  &  9,735\\
  Kung-Fu Master  &   258.5  &   22,736.3  &   26,059.0  &   29,710.0  &   34,294.0  &  \textbf{\textcolor{blue}{ 48,375.0}}  &  48,192\\
  Montezuma's Revenge  &   0.0  &  \textbf{\textcolor{blue}{ 4,753.3}}  &   0.0  &   0.0  &   0.0  &   0.0  &  0.0\\
  Ms. Pac-Man  &   307.3  &  \textbf{\textcolor{blue}{ 6,951.6}}  &   3,085.6  &   2,711.4  &   6,283.5  &   3,327.3  &  3,415\\
  Name This Game  &   2,292.3  &   8,049.0  &   8,207.8  &   10,616.0  &   11,971.1  &  \textbf{\textcolor{blue}{ 15,572.5}}  &  12,542\\
  Phoenix  &   761.4  &   7,242.6  &   8,485.2  &   12,252.5  &   23,092.2  &  \textbf{\textcolor{blue}{ 70,324.3}}  &  17,490\\
  Pitfall!  &   -229.4  &  \textbf{\textcolor{blue}{ 6,463.7}}  &   -286.1  &   -29.9  &   0.0  &   0.0  &  0.0\\
  Pong  &   -20.7  &   14.6  &   19.5  &   20.9  &  \textbf{\textcolor{blue}{ 21.0}}  &   20.9  &  20.9\\
  Private Eye  &   24.9  &  \textbf{\textcolor{blue}{ 69,571.3}}  &   146.7  &   129.7  &   103.0  &   206.0  &  15,095\\
  Q*Bert  &   163.9  &   13,455.0  &   13,117.3  &   15,088.5  &   19,220.3  &   18,760.3  &  \textbf{\textcolor{blue}{23,784}}\\
  River Raid  &   1,338.5  &   17,118.0  &   7,377.6  &   14,884.5  &  \textbf{\textcolor{blue}{ 21,162.6}}  &   20,607.6  &  17,322\\
  Road Runner  &   11.5  &   7,845.0  &   39,544.0  &   44,127.0  &  \textbf{\textcolor{blue}{ 69,524.0}}  &   62,151.0  &  55,839\\
  Robotank  &   2.2  &   11.9  &   63.9  &   65.1  &  \textbf{\textcolor{blue}{ 65.3}}  &   27.5  &  52.3\\
  Seaquest  &   68.4  &   42,054.7  &   5,860.6  &   16,452.7  &   50,254.2  &   931.6  &  \textbf{\textcolor{blue}{266,434}}\\
  Skiing  &   -17,098.1  &  \textbf{\textcolor{blue}{ -4,336.9}}  &   -13,062.3  &   -9,021.8  &   -8,857.4  &   -19,949.9  &  -13,901\\
  Solaris  &   1,236.3  &  \textbf{\textcolor{blue}{ 12,326.7}}  &   3,482.8  &   3,067.8  &   2,250.8  &   133.4  &  8,342\\
  Space Invaders  &   148.0  &   1,668.7  &   1,692.3  &   2,525.5  &   6,427.3  &  \textbf{\textcolor{blue}{ 15,311.5}}  &  5,747\\
  Star Gunner  &   664.0  &   10,250.0  &   54,282.0  &   60,142.0  &   89,238.0  &  \textbf{\textcolor{blue}{ 125,117.0}}  &  49,095\\
  Surround  &   -10.0  &   6.5  &   -5.6  &   -2.9  &   4.4  &   1.2  &  \textbf{\textcolor{blue}{6.8}}\\
  Tennis  &   -23.8  &   -8.3  &   12.2  &   -22.8  &   5.1  &   0.0  &  \textbf{\textcolor{blue}{23.1}}\\
  Time Pilot  &   3,568.0  &   5,229.2  &   4,870.0  &   8,339.0  &  \textbf{\textcolor{blue}{ 11,666.0}}  &   7,553.0  &  8,329\\
  Tutankham  &   11.4  &   167.6  &   68.1  &   218.4  &   211.4  &   245.9  &  \textbf{\textcolor{blue}{280}}\\
  Up and Down  &   533.4  &   11,693.2  &   9,989.9  &   22,972.2  &  \textbf{\textcolor{blue}{ 44,939.6}}  &   33,879.1  &  15,612\\
  Venture  &   0.0  &   1,187.5  &   163.0  &   98.0  &   497.0  &   48.0  &  \textbf{\textcolor{blue}{1,520}}\\
  Video Pinball  &   16,256.9  &   17,667.9  &   196,760.4  &   309,941.9  &   98,209.5  &   479,197.0  &  \textbf{\textcolor{blue}{949,604}}\\
  Wizard Of Wor  &   563.5  &   4,756.5  &   2,704.0  &   7,492.0  &   7,855.0  &  \textbf{\textcolor{blue}{ 12,352.0}}  &  9,300\\
  Yars' Revenge  &   3,092.9  &   54,576.9  &   18,098.9  &   11,712.6  &   49,622.1  &  \textbf{\textcolor{blue}{ 69,618.1}}  &  35,050\\
  Zaxxon  &   32.5  &   9,173.3  &   5,363.0  &   10,163.0  &   12,944.0  &  \textbf{\textcolor{blue}{ 13,886.0}}  &  10,513
\end{tabular}
\caption{Raw scores across all games, starting with 30 no-op actions. Reference values from \citet{wang2016dueling}.\label{fig:atari_sota}}
\end{figure*}

\clearpage

\newpage

\begin{figure*}[htb]
\begin{center}
\includegraphics[width=0.5\textwidth]{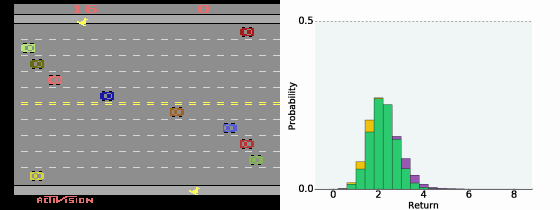}
\hspace{-1em}\includegraphics[width=0.5\textwidth]{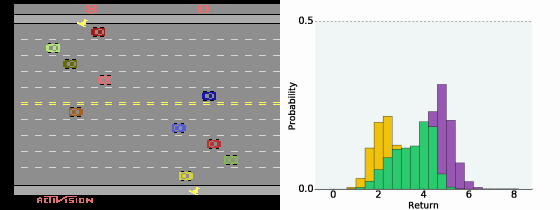}
\caption{\textsc{Freeway}: Agent differentiates action-value distributions under pressure.\label{fig:atari_frames0}}
\end{center}
\end{figure*}

\begin{figure*}[htb]
\begin{center}
\includegraphics[width=0.5\textwidth]{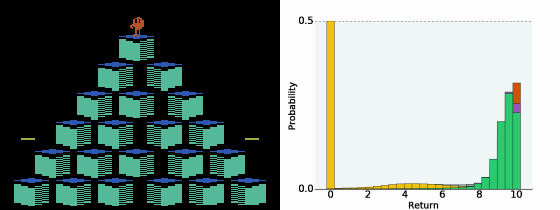}
\hspace{-1em}\includegraphics[width=0.5\textwidth]{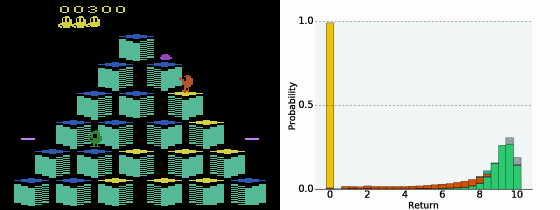}\\
\includegraphics[width=0.5\textwidth]{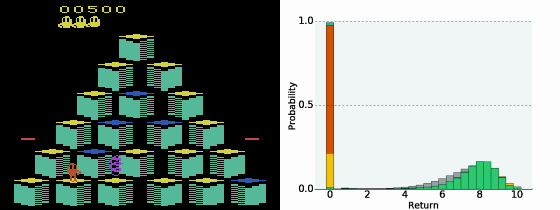}
\hspace{-1em}\includegraphics[width=0.5\textwidth]{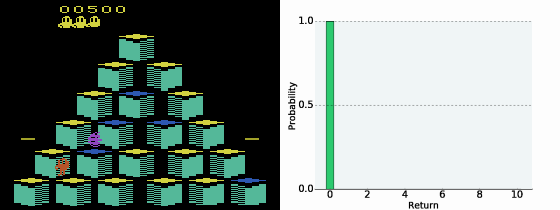}
\caption{\textsc{Q*Bert}: Top, left and right: Predicting which actions are unrecoverably fatal. Bottom-Left: Value distribution shows steep consequences for wrong actions. Bottom-Right: The agent has made a huge mistake.}
\end{center}
\end{figure*}

\begin{figure*}[htb]
\begin{center}
\includegraphics[width=0.33\textwidth]{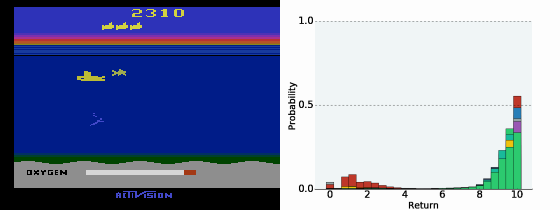}
\hspace{-1em}\includegraphics[width=0.33\textwidth]{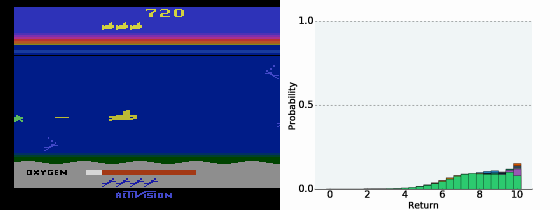}
\hspace{-1em}\includegraphics[width=0.33\textwidth]{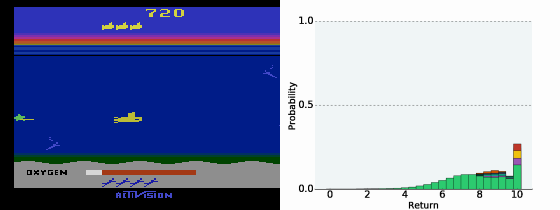}
\caption{\textsc{Seaquest}: Left: Bimodal distribution. Middle: Might hit the fish. Right: Definitely going to hit the fish.}
\end{center}
\end{figure*}

\begin{figure*}[htb]
\begin{center}
\includegraphics[width=0.5\textwidth]{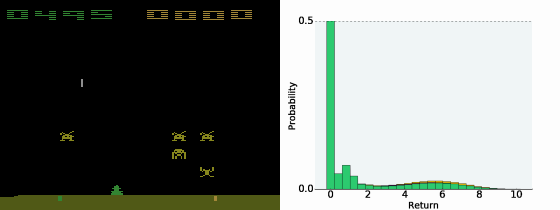}
\hspace{-1em}\includegraphics[width=0.5\textwidth]{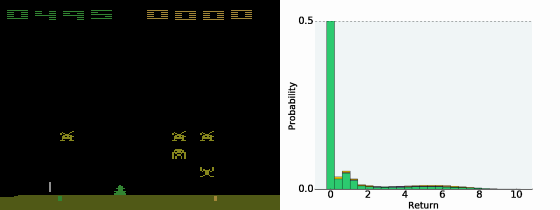}\\
\includegraphics[width=0.33\textwidth]{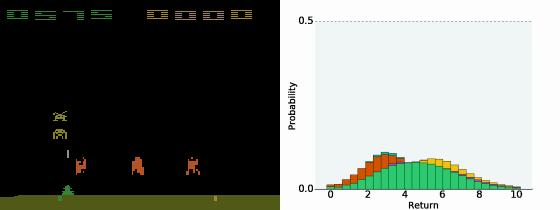}
\hspace{-1em}\includegraphics[width=0.33\textwidth]{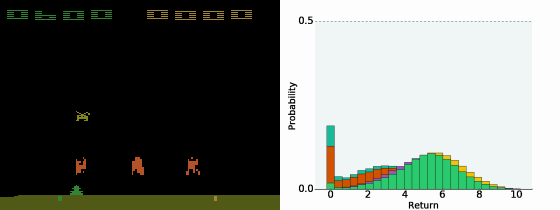}
\hspace{-1em}\includegraphics[width=0.33\textwidth]{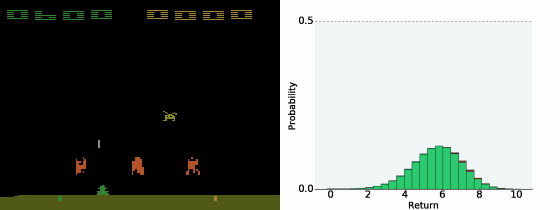}
\caption{\textsc{Space Invaders}: Top-Left: Multi-modal distribution with high uncertainty. Top-Right: Subsequent frame, a more certain demise. Bottom-Left: Clear difference between actions. Bottom-Middle: Uncertain survival. Bottom-Right: Certain success.\label{fig:atari_frames4}}
\end{center}
\end{figure*}

\end{document}